\documentclass[journal]{IEEEtai}

\usepackage[colorlinks,urlcolor=blue,linkcolor=blue,citecolor=blue]{hyperref}

\usepackage{color,array}
\usepackage{amsmath, amssymb}
\usepackage{graphicx}
\usepackage{booktabs}
\usepackage{algorithm, algorithmicx, algpseudocode}
\usepackage{comment}
\usepackage{bm}
\usepackage{multirow}
\usepackage{subcaption} % Ensure this is in your preamble
\newtheorem{theorem}{Theorem}

\newtheorem{corollary}{Corollary}
\newtheorem{proposition}{Proposition}

\newcommand{\autoformer}{\texttt{Autoformer}}
\newcommand{\informer}{\texttt{Informer}}
\newcommand{\patchtst}{\texttt{PatchTST}}
\newcommand{\fedformer}{\texttt{FEDformer}}
\newcommand{\pyraformer}{\texttt{Pyraformer}}

\newcommand{\dkoopformer}{\texttt{DeepKoopFormer}}

\newcommand{\sigmoid}{\sigma}
\newcommand{\K}{\mathcal K}

\newcommand{\ssdkf}{\texttt{scalar-gated}}
\newcommand{\pmcdkf}{\texttt{per-mode gated}}
\newcommand{\mlpdkf}{\texttt{MLP-shaped} spectral mapping}
\newcommand{\lrdkf}{\texttt{low-rank} Koopman}

\newcommand{\dlinear}{\texttt{DLinear}}

\newcommand{\ldkf}{\texttt{Learnable-DeepKoopFormer}}

\usepackage{bm} % in preamble

\usepackage{hyperref}

\usepackage{bm}

\begin{document}

\onecolumn
\large
This work has been submitted to the IEEE for possible publication.
Copyright may be transferred without notice, after which this version
may no longer be accessible.
\twocolumn
\normalsize

\title{Learnable Koopman-Enhanced Transformer-Based Time Series Forecasting with Spectral Control}

\author{Ali Forootani*, \IEEEmembership{Senior Member,~IEEE,} Raffaele Iervolino, \IEEEmembership{Senior Member,~IEEE,}
        % <-this % stops a space
\thanks{Ali Forootani is with Max Planck Institute of Geoanthropology Kahlaische Str. 10, 07745 Jena, Germany (\texttt{email:}forootani@gea.mpg.de/aliforootani@ieee.org).}% <-this % stops a space
\thanks{Raffaele Iervolino is with Department of Electrical Engineering and Information Technology, University of Naples, 80125 Napoli, Italy(\texttt{email:}rafierv@unina.it).}

%\thanks{ Mohammad Khosravi is with Delft Center for Systems and Control, Mekelweg 2, Delft, 2628 CD, The Netherlands, \texttt{Email: mohammad.khosravi@tudelft.nl}}
%\thanks{Masoud Barati is with School of Engineering, University of Pittsburgh Swanson, O'Hara Street Benedum Hall of Engineering, Pittsburgh, 3700, Pennsylvania, United States\texttt{Email: masoud.barati@pitt.edu}}
}

%\markboth{Journal of IEEE Transactions on Artificial Intelligence, Vol. 00, No. 0, Month 2020}
%{First A. Author \MakeLowercase{\textit{et al.}}: Bare Demo of IEEEtai.cls for IEEE Journals of IEEE Transactions on Artificial Intelligence}

\maketitle

\begin{abstract}
This paper proposes a unified family of learnable Koopman operator parameterizations that integrate linear dynamical systems theory with modern deep learning forecasting architectures. We introduce four learnable Koopman variants—\ssdkf, \pmcdkf, \mlpdkf, and \lrdkf~operators—which generalize and interpolate between strictly stable Koopman operators and unconstrained linear latent dynamics. Our formulation enables explicit control over the spectrum, stability, and rank of the linear transition operator while retaining compatibility with expressive nonlinear backbones such as \patchtst, \autoformer, and \informer.

We evaluate the proposed operators in a large-scale benchmark that also includes LSTM, DLinear, and simple diagonal State-Space Models (SSMs), as well as lightweight transformer variants. Experiments across multiple horizons and patch lengths show that learnable Koopman models provide a favorable bias–variance trade-off, improved conditioning, and more interpretable latent dynamics. We provide a full spectral analysis, including eigenvalue trajectories, stability envelopes, and learned spectral distributions. Our results demonstrate that learnable Koopman operators are effective, stable, and theoretically principled components for deep forecasting.
\end{abstract}


%%%%%%%%%%%%%%%%%%%%%%%%%%%%%%%%%%%%%%%%%%%%%%%%%%%
%%%%%%%%%%%%%%%%%%%%%%%%%%%%%%%%%%%%%%%%%%%%%%%%%%%
%%%%%%%%%%%%%%%%%%%%%%%%%%%%%%%%%%%%%%%%%%%%%%%%%%%

\section{Introduction}
\label{sec:intro}

Time series forecasting is a foundational task across science and engineering, underpinning applications in retail demand prediction \cite{Croston1972,wen2017multi,salinas2020deepar}, traffic flow management \cite{lv2014traffic,li2018diffusion}, energy balancing \cite{DIMOULKAS2019,Saxena2019,forootani_cadnn_2025}, and financial volatility modeling \cite{callot2017modeling,neurips_18_long_tail}. The increasing scale, dimensionality, and non-stationarity of modern datasets challenge traditional statistical models \cite{hyndman2018forecasting}, driving the development of deep learning architectures that offer improved expressiveness, scalability, and feature abstraction.

Early deep learning models such as RNNs and LSTMs \cite{hochreiter1997long,salinas2020deepar}, as well as CNN and TCN architectures \cite{lecun1995convolutional,BaiTCN2018}, demonstrated strong forecasting capabilities by learning temporal dependencies directly from data. More recently, Transformer-based architectures \cite{wu2021autoformer,haoyietal-informer-2021,nie2022time,vaswani2017attention} have become the dominant paradigm due to their ability to model long-range dependencies via self-attention. Their effectiveness has been validated in large-scale forecasting competitions such as M4 and M5 \cite{smyl2018m4,makridakis2021m5}.

To improve scalability and accuracy, numerous specialized Transformers have been proposed. \informer~\cite{informer,haoyietal-informer-2021} introduces \texttt{ProbSparse} attention to focus on the most informative queries; \autoformer~\cite{autoformer,wu2021autoformer} integrates decomposition blocks to explicitly handle trend and seasonal patterns; \fedformer~\cite{fedformer,zhou2022fedformer} employs Fourier spectral mixing to reduce complexity; and \pyraformer~\cite{pyraformer,liu2021pyraformer} introduces a pyramidal hierarchy to capture multi-scale structures. More recently, \patchtst~\cite{nie2022time} showed that simple architectural principles---patching, channel-independence, and lightweight attention---are sufficient to outperform many deeper Transformer variants \cite{multichannel,dlinear}.

Following the surprising results of \dlinear~\cite{dlinear,zeng2023transformers}, which showed that a channel-wise linear model can outperform many Transformers on standard benchmarks, a central question has emerged: \emph{how much architectural complexity is actually necessary for strong time-series forecasting performance?} Moreover, deep models still suffer from limited interpretability, sensitivity to distribution shifts \cite{kuznetsov2020discrepancy,Liu2022NonstationaryTR,kim2022reversible}, and a lack of explicit dynamical structure, limiting their robustness in high-stakes settings.

To address these issues, recent research has sought to embed domain priors and system dynamics into learning architectures. Hybrid statistical–deep models \cite{wang2019deepfactors,sen2019think,smyl2018m4}, continuous-time neural ODE frameworks \cite{chen19,vialard20}, and physics-inspired models aim to introduce structural inductive biases. However, many such approaches still require strong assumptions on fixed dynamics and face scalability challenges in multivariate or highly nonlinear environments.

\subsection{Koopman Operators and Deep Learning}

The \textit{Koopman operator} offers a linear lens for analyzing nonlinear dynamical systems by evolving observables rather than states. Formally, it is an infinite-dimensional linear operator acting on functions of the state, enabling the use of linear spectral tools in settings that may be nonlinear, nonstationary, or even chaotic \cite{Mezic2017book, khosravi2023representer}. Classical data-driven approaches such as Dynamic Mode Decomposition (DMD) approximate the Koopman spectrum for prediction and modal analysis \cite{williams2015data,Kutz2016book}, but tend to degrade under high-dimensional, noisy, or multivariate real-world data. Dynamic Mode Decomposition (DMD) offer practical means to approximate the operator from measurements and facilitate modal analysis, prediction, and control in complex systems~\cite{Drgona2022,Skomski2021,NEURIPS2021_c9dd73f5}.

Deep learning has revitalized Koopman theory by learning nonlinear embeddings in which latent dynamics evolve linearly. Neural Koopman models typically employ an encoder $\mathcal{E}_\theta$ to produce latent coordinates $z_t=\mathcal{E}_\theta\big(X_t\big)$ from an input window $X_t$, and a linear propagator $ z_{t+1}=\K  z_t$, often coupled with nonlinear decoders, spectral regularization, and stability constraints \cite{Lusch2018,Yeung2019,NIPS2017_3a835d32,lusch2018deep,azencot2020forecasting,morton2018deep,takeishi2017learning}. These advances have led to interpretable and stable rollout predictions, with recent extensions exploring low-rank structure, symmetry priors, and multiresolution latent spaces \cite{brunton2021modern}. However, most architectures rely on shallow networks and struggle to model long-range, non-local temporal dependencies that modern Transformer-based models address effectively.

%%%%%%%%%%%%%%%%%%%%%%%%%%%%%%%%%%%%%%%%%%%%%%%%%%%%

%\subsection{Koopman–Transformer Hybrids}

Recent work has begun incorporating Koopman operators into Transformer architectures to obtain stable and interpretable latent dynamics, so called \dkoopformer    \cite{forootani2025deepkoopformer}.

Compared to existing Koopman-based forecasting models \cite{lusch2018deep,takeishi2017learning,azencot2020forecasting}, \dkoopformer~is the first to offer provable spectral decay, Lyapunov-stable propagation, normality of the transition matrix, and closed-form error control—all within a general, extensible encoder–propagator–decoder pipeline aligned with the goals of physics-informed forecasting \cite{PIML2021,Truong2023} and interpretable dynamical systems modeling \cite{Brunton2016plosone,Korda2016arxiv}.

%These hybrid models aim to combine (i) the expressive long-range modeling capacity of Transformers, (ii) the stability and interpretability of operator-theoretic dynamics, (iii) robustness under distributional shifts, and (iv) explicit analysis through Koopman spectra and modes.

%Compared to existing Koopman-based forecasting models \cite{lusch2018deep,takeishi2017learning,azencot2020forecasting}, \dkoopformer~is the first to offer provable spectral decay, Lyapunov-stable propagation, normality of the transition matrix, and closed-form error control—all within a general, extensible encoder–propagator–decoder pipeline aligned with the goals of physics-informed forecasting \cite{PIML2021,Truong2023} and interpretable dynamical systems modeling \cite{Brunton2016plosone,Korda2016arxiv}.

%%%%%%%%%%%%%%%%%%%%%%%%%%%%%%%%%%%%%%%%%%%%%%%%

\subsection{Contributions: Learnable Koopman--Transformer Framework}

Although initial results demonstrate strong potential, most existing designs rely on fixed or heavily constrained Koopman operators, leaving a large space of learnable operator parameterizations largely unexplored.
Indeed, existing Koopman–neural models typically combine nonlinear feature encoders with a latent linear operator \cite{takeishi2017learning,lusch2018deep}. However, the operator itself is usually learned \emph{without explicit spectral parameterization}: it is fitted as an unconstrained linear layer or least-squares estimator, providing no mechanism to shape contraction, oscillation, or multi-scale temporal behavior. Moreover, such formulations have rarely been integrated with modern Transformer architectures, despite their demonstrated ability to model long-range dependencies in time series. This combination of \emph{spectral non‐control} and \emph{architectural underuse} limits both interpretability and dynamical expressiveness.

To address these limitations, we introduce the \ldkf{}, a unified framework that equips Transformer backbones with \emph{parameterized Koopman operators} whose spectra are learned directly from data. Our formulation includes four operator families:

\begin{itemize}
    \item \ssdkf: global trainable parameters $(\alpha,\beta)$ control spectral shifting and scaling, enabling tunable damping or persistent dynamics;
    \item  \pmcdkf: dimension-wise parameters $(\alpha_i,\beta_i)$ model anisotropic temporal responses and multi-frequency evolution;
    \item \mlpdkf: neural spectral mappings are projected onto stable linear propagators with controlled spectral radius;
    \item \lrdkf: relatively low rank factorizations of the Koopman operator to capture high-dimensional latent dynamics efficiently through structured compression.
\end{itemize}

These operator designs are embedded into lightweight and full Transformer forecasters, including \patchtst{} \cite{nie2022time}, \informer{} \cite{haoyietal-informer-2021}, and \autoformer{} \cite{wu2021autoformer}, yielding a comprehensive family of \dkoopformer ~models. Spectral stability constraints and operator regularization allow latent dynamics to be jointly optimized with representation learning while preserving interpretability: the dominant modes, stability margins, and spectral envelopes remain directly analyzable.

%We conduct an extensive benchmark spanning constrained, learnable, and unconstrained Koopman variants \color{red}[ACTUALLY, in the paper only the learnable variants are detailed, the "constrained" and the "unconstrained" have not been defined]\color{black}, compared against widely used baselines: LSTM, \dlinear{} \cite{dlinear,zeng2023transformers}, and a diagonal State-Space Model (SSM) \cite{gu2021combining}. Alongside forecasting performance, we provide the first systematic \textit{spectral evaluation} of Koopman–Transformers, including density characterization, training trajectories of eigenvalues, stability envelopes, and horizon sensitivity.

We conduct an extensive simulation benchmark comparing the proposed learnable Koopman variants against two reference Koopman baselines---a spectrally constrained variant and an unconstrained variant---together with widely used forecasting models including LSTM, \dlinear~\cite{dlinear,zeng2023transformers}, and a diagonal State-Space Model (SSM)~\cite{gu2021combining}. Alongside forecasting accuracy, we provide the first systematic \emph{spectral evaluation} of Koopman--Transformers, including density characterization, eigenvalue training trajectories, stability envelopes, and horizon sensitivity.\\
In summary, this work provides:
\begin{itemize}
    \item a unified parameterization of \textit{learnable Koopman operators} (scalar, per-mode, neural, and low-rank);
    \item the \ldkf~framework, embedding these operators into three state-of-the-art Transformer backbones;
    \item a systematic benchmark of \textit{21 models} across constrained, learnable, and unconstrained variants, plus LSTM, \dlinear, and diagonal SSM baselines;
    \item the first comprehensive \textit{spectral analysis of Koopman–Transformer models}, revealing stability properties, multi-scale expressiveness, and interpretable latent dynamics.
\end{itemize}

%%%%%%%%%%%%%%%%%%%%%%%%%%%%%%%%%%%%%%%%%%%%%%%%%%%%%

For real-world evaluation, we utilize high-dimensional datasets from multiple domains, including: the CMIP6 climate projections~\footnote{\url{https://cds.climate.copernicus.eu/datasets/projections-cordex-domains-single-levels?tab=overview}} and ERA5 reanalysis data~\footnote{\url{https://cds.climate.copernicus.eu/datasets}}, focusing on wind speed and surface pressure forecasting over Germany; a financial time series dataset~\footnote{\url{https://github.com/Chisomnwa/Cryptocurrency-Data-Analysis}} for cryptocurrency market analysis; and an electricity generation dataset~\footnote{\url{https://github.com/afshinfaramarzi/Energy-Demand-electricity-price-Forecasting/tree/main}} for modeling energy supply dynamics. These datasets collectively span chaotic, periodic, and stochastic regimes, allowing for a comprehensive assessment of model accuracy, stability, and generalization across domains.

Compared to existing Koopman-based forecasting models \cite{lusch2018deep,takeishi2017learning,azencot2020forecasting}, \ldkf~is the first to offer provable spectral decay, Lyapunov-stable propagation, normality of the transition matrix, and closed-form error control—all within a general, extensible encoder–propagator–decoder pipeline aligned with the goals of physics-informed forecasting \cite{PIML2021,Truong2023} and interpretable dynamical systems modeling \cite{Brunton2016plosone,Korda2016arxiv}.

This paper is organized as follows. Section~\ref{preliminaries_sec} provides the preliminaries required for time series forecasting. In Section~\ref{sec:learnable_koop_math}, we present \ldkf~architectures for multivariate time series forecasting. Numerical simulations are presented in Section~\ref{sec:simulations} to evaluate the performance of \ldkf versus other benchmarks. Finally, we conclude the article in Section~\ref{sec:conclusion}.

%%%%%%%%%%%%%%%%%%%%%%%%%%%%%%%%%%%%%%%%%%%%%%%%%%%
%%%%%%%%%%%%%%%%%%%%%%%%%%%%%%%%%%%%%%%%%%%%%%%%%%%
%%%%%%%%%%%%%%%%%%%%%%%%%%%%%%%%%%%%%%%%%%%%%%%%%%%

\section{\ldkf~Architecture}
\label{preliminaries_sec}

We begin by formalizing the multivariate time series forecasting task and 
introducing the key components of the proposed architecture.

Given a multivariate time series $\{x_t\}_{t=1}^T$, with $x_t \in\mathbb{R}^{d_s}$,  \ldkf~operates on context windows of length $P$ and predicts the subsequent $H$ samples. For each valid index $t$, an input segment $X_t = [x_t, \dots, x_{t+P-1}] \in \mathbb{R}^{P \times d_s}$ is paired with its future sequence $Y_t = [x_{t+P}, \dots, x_{t+P+H-1}] \in \mathbb{R}^{H \times d_s}$. The input window is passed through a Transformer encoder with positional encoding \(\mathcal{E}_\theta\) (e.g., \patchtst, \autoformer, or \informer), producing a latent representation \[ z_t = \mathcal{E}_\theta(X_t)\in\mathbb{R}^{d_{\mathrm{lat}}},
\]
where \(d_{\mathrm{lat}}\) is the latent dimension. Typically \(d_{\mathrm{lat}}=d_{\mathrm{model}}\),
and not necessarily equal to the physical dimension \(d_s\), where \(d_{\text{model}}\) denotes the hidden (embedding) dimension of the Transformer backbone, i.e., the dimensionality of the token representations and of the latent state produced by the encoder \(\mathcal{E}_\theta\). With the slight abuse of notation and for the sake of simplicity, hereinafter we assume $d_{\mathrm{model}}=d_{\mathrm{lat}}=d_s=d$.

Temporal evolution in the latent space is modeled through a learned Koopman operator $\K_\phi$, which advances the latent state linearly according to $ z_{t+1} = \K_\phi z_t$. The propagated latent state is decoded into a direct $H$-step forecast using a linear mapping $\mathcal{D}_\varphi$, yielding a vector $\hat{y}_t = \mathcal{D}_\varphi(z_{t+1}) \in \mathbb{R}^{H \cdot d}$, which is then reshaped into the output sequence $\hat{Y}_t \in \mathbb{R}^{H \times d}$.

Here, the subscript \(\theta\)  denotes the parameters of the Transformer encoder \(\mathcal{E}_\theta\), \(\phi\) collects the parameters of the Koopman propagator \(\K_\phi\) (including its spectral coefficients and orthogonal factors), and \(\varphi\) denotes the parameters of the linear decoder \(\mathcal{D}_\varphi\).  During inference, a context window \(X_t\) is encoded once, propagated through the Koopman operator \(\K_\phi\), and decoded to produce the entire prediction horizon \(H\) in a single forward pass, without autoregressive rollout.

All components are trained jointly in an end-to-end fashion using the Adam optimizer \cite{kingma2014adam},
a first-order stochastic gradient method with adaptive moment estimation. In particular, training is performed end-to-end by minimizing a forecasting loss (e.g., mean squared error) augmented with a Lyapunov-inspired penalty as follows 
\begin{align}\label{LyapEn}
    \mathcal{L}
    ~=~
    &\mathbb{E} \Big[ \Vert \widehat{Y}_t - Y_t \Vert^2 \Big]
    ~+~
    \lambda_{\mathrm{Lyap}}
    \,\mathbb{E}\Big[
      \bigl\lVert \K_\phi {\bm z}_t \bigr\rVert_P^2
      - \bigl\lVert {z}_t \bigr\rVert_P^2
    \Bigr]_+ \nonumber \\
    &=\mathcal{L}_{\mathrm{MSE}}
    ~+~
    \lambda_{\mathrm{Lyap}}\mathcal{L}_{\mathrm{Lyap}},
\end{align}
where $\|z\|_P^2 = z^\top P z$, $P \succ 0$ is a fixed positive-definite matrix, $[x]_+ = \max\{x,0\}$, and $\lambda_{\text{Lyap}} > 0$ is a scalar regularization weight controlling the strength of the Lyapunov stability penalty. %The second term penalizes unstable latent growth and promotes Koopman-consistent dynamics.\color{black}

%\begin{equation}\label{minim}
%    \mathcal{L} = \Vert \widehat{Y}_t - Y_t \Vert^2 + \lambda\, \mathrm{ReLU}\!\left(\Vert z_{t+1} \Vert^2 - \Vert z_t \Vert^2\right),
%\end{equation}
%where the second term penalizes unstable latent growth and promotes Koopman-consistent dynamics. 

%All parameters $(\theta, \phi, \K)$ \color{red}[THE parameters $\theta,\phi$ are not clearly defined]\color{black} are optimized jointly with Adam \color{red}[PUT a ref or explain what is it]\color{black}. During inference, a context window $X_t$ is encoded once, propagated through $\K$, and decoded, producing the entire horizon $H$ in a single forward pass with no autoregressive rollout.

%%%%%%%%%%%%%%%%%%%%%%%%%%%%%%%%%%%%%%%
%%%%%%%%%%%%%%%%%%%%%%%%%%%%%%%%%%%%%%%
%%%%%%%%%%%%%%%%%%%%%%%%%%%%%%%%%%%%%%%

\section{Learnable Koopman~Variants}
\label{sec:learnable_koop_math}

%Consider a multivariate time series \(\{x_t\}_{t=1}^T\) with \(x_t\in\mathbb{R}^{d_s}\),
%and form at each time index \(t\) a \emph{past} context window of length \(P\),
%\begin{align}
%X_t :&= x_{t-P+1:t} = [x_{t-P+1},\ldots,x_t]\in\mathbb{R}^{P\times d_s},
%Y_t :&= x_{t+1:t+H} = [x_{t+1},\ldots,x_{t+H}]\in\mathbb{R}^{H\times d_s}.
%\end{align}
%The encoder backbone \(\mathcal{E}_\theta\) (e.g., \patchtst, \autoformer, or \informer)
%maps the context to a latent state
%\[ z_t = \mathcal{E}_\theta(X_t)\in\mathbb{R}^{d_{\mathrm{lat}}},
%\]
%where \(d_{\mathrm{lat}}\) is the latent dimension. 
In \ldkf~we assume the latent evolution to be
modeled by a linear dynamics through a learned Koopman propagator \(\K_\phi\in\mathbb{R}^{d \times d }\),
\[
z_{t+1} = \K_\phi z_t.
\]
As shown later in the paper, to ensure identifiability and stable training, $\K_\phi$ is parameterized through
an orthogonal–diagonal–orthogonal (ODO) factorization
\begin{equation}
    \K_\phi \;=\;
    U_\phi \, \mathrm{diag}\!\left(\Sigma_\phi\right) 
    V_\phi^\top ,
    \label{eq:koop_odo}
\end{equation}
where
$U_\phi, V_\phi$ are learned orthonormal matrices obtained via
QR retraction, and
$\Sigma_\phi = (\Sigma_{\phi,1},\ldots,\Sigma_{\phi,d})^\top \in\mathbb{R}^d$ (or $\mathbb{R}^r$ in the low-rank
case, see later in the paper) the vector containing the learnable spectral coefficients.
Equation \eqref{eq:koop_odo} ensures that all nonlinearity and expressiveness
are concentrated in the spectrum $\Sigma_\phi$, while the
left/right singular directions evolve on the Stiefel manifold, i.e., the set of matrices with orthonormal columns \cite{absil2008optimization}. 

%Finally, the decoder \(\mathcal{D}_\varphi\) produces the
%multi-step forecast in a single forward pass,
%\[
%\widehat{Y}_t = D_\varphi(z_{t+1})\in\mathbb{R}^{H\times d},
%\]
%and the model is trained end-to-end by minimizing a forecasting loss (e.g., MSE)
%augmented with a Lyapunov-inspired penalty that discourages latent energy growth.
The key design freedom of our \emph{learnable Koopman architecture} lies in how
the diagonal spectral coefficients $\Sigma_\phi$ are generated.
We consider Koopman variants that share the principle
\begin{equation}
    \Sigma_{\phi,i} ~=~ \rho_{\max} \,
    \sigma_\phi\!\big(S_i\big),
    \qquad i = 1,\dots,d,
    \label{eq:spectrum_general}
\end{equation}
where $S_i\in\mathbb{R}$ are raw trainable parameters, \(\rho_{\max}\in(0,1)\) is a prescribed spectral bound, and
$\sigma_\phi(\cdot)$ is a smooth squashing map, i.e. a nonlinear “spectral shaping” function, which preserves differentiability and enables end-to-end learning while providing explicit control over contraction rates in the latent Koopman dynamics. This map is chosen to enforce 
\begin{equation}
    \lvert \Sigma_{\phi,i} \rvert < \rho_{\max}
    \quad\text{for all } i,
    \qquad
    \rho_{\max} \in (0,1),
    \label{eq:rho_max_bound}
\end{equation}
whenever spectral stability is desired, while allowing different degrees of
expressiveness.  
In particular, it is the composition of the logistic sigmoid  \(\sigma(x)=\frac{1}{1+e^{-x}}\) and a function chosen from a list of spectral shaping maps. More specifically, we introduce four learnable Koopman families that act on the vector of raw spectral parameters 
\begin{enumerate}
    \item  \ssdkf.  
    A shared affine gate controls the entire spectrum:
    \begin{equation}
        \sigma_\phi(S_i)
        ~=~
        \sigma(\alpha S_i + \beta),
    \end{equation}
    with $\alpha,\beta\in\mathbb{R}$ learned scalars.

    \item \pmcdkf. 
    Each latent direction receives its own gate,
    \begin{equation}
        \sigma_\phi(S_i)
        ~=~
        \sigma(\alpha_i S_i + \beta_i),
    \end{equation}
    allowing anisotropic amplification/decay across modes.

    \item \mlpdkf.  
    A small neural operator $g_\phi:\mathbb{R}\to\mathbb{R}$ transforms
    raw spectral parameters before squashing:
    \begin{equation}
        \sigma_\phi(S_i)
        ~=~
        \sigma\!\left(g_\phi(S_i)\right),
    \end{equation}
    enabling flexible nonlinear reshaping of the spectrum.

    \item \lrdkf.
    The operator is restricted to rank $r \ll d$ via  
\begin{equation}
    \K_\phi
    ~=~
    U_{r,\phi}
    \,\mathrm{diag}(\Sigma_{r,\phi})\,
    V_{r,\phi}^\top,
    \quad
    U_{r,\phi}, V_{r,\phi} \in \mathbb{R}^{d\times r},
    \label{eq:koop_lowrank}
\end{equation}
with orthonormal columns and a spectral vector
$\Sigma_{r,\phi} \in \mathbb{R}^r$ satisfying again
$\lvert \Sigma_{r,\phi,i} \rvert < \rho_{\max}$.  It encourages low-dimensional latent dynamics and reduces complexity.
\end{enumerate}

%Here $\sigma(x)$ denotes the logistic sigmoid,
%and $\sigma_\phi$ denotes the variant-specific spectral shaping map
%built from it, i.e.
%\begin{equation}\label{eq:sig_func}
%\Sigma_{\phi,i}
%~=~ \rho_{\max}\,\sigma\!\big( h_\phi(S_i) \big),
%\qquad
%\sigma(x)=\frac{1}{1+e^{-x}} .
%\end{equation}
%\begin{equation}\label{eq:func_choice}
%h_\phi(S_i) \;=\;
%\begin{cases}
%\alpha S_i + \beta, & \text{\ssdkf },\\[0.35em]
%\alpha_i S_i + \beta_i, & \text{\pmcdkf},\\[0.35em]
%g_\phi(S_i), & \text{\mlpdkf},\\[0.35em]
%\alpha S_i + \beta,\; i \le r, & \text{\lrdkf}.
%\end{cases}
%\end{equation}

\noindent In addition, for comparison purpose, we also consider a fully free operator (unconstrained baseline) obtained when $\K_\phi$ is a free dense matrix,
without the ODO structure \eqref{eq:koop_odo} and without the spectral bound
\eqref{eq:rho_max_bound}.
This variant maximizes expressiveness but lacks the stability and dynamical
bias of the structured families above.

%All variants are trained end-to-end with a standard forecasting loss
%$\mathcal{L}_{\mathrm{MSE}}$ augmented by a Lyapunov-inspired stability term
%penalizing increases in quadratic energy
%\begin{equation}\label{LyapEn}
%    \mathcal{L}
%    ~=~
%    \mathcal{L}_{\mathrm{MSE}}
%    ~+~
%    \lambda_{\mathrm{Lyap}}\mathcal{L}_{\mathrm{Lyap}},
%\end{equation}
%where $\lambda_{\text{Lyap}} > 0$ is a scalar regularization weight controlling the strength of the Lyapunov stability penalty. This term promotes stable latent evolution and provides a common regularizer across constrained, learnable, and unconstrained Koopman variants, which further biases optimisation toward strict contraction of latent energy.

Overall, the learnable Koopman architecture provides a continuum between
rigorously constrained operator-theoretic propagators and fully flexible
data-driven linear maps, allowing us to isolate and study the role of spectral
structure in Transformer-based forecasting models.

\subsection{Theoretical Properties of the Learnable Koopman Propagator}
\label{sec:theory_koop}

We consider latent states $\bm z_t \in \mathbb{R}^d$ evolving under a linear
propagator
\begin{equation}\label{latentdyn}
     z_{t+1} = \K_\phi  z_t,
    \qquad
    \K_\phi \in \mathbb{R}^{d \times d},
\end{equation}
where $\K_\phi$ is parameterized by the learnable weights $\phi$ of the
Koopman module in \dkoopformer. This subsection collects several structural
properties of the proposed parameterizations, focusing on stability,
expressiveness, low-rank structure, and Lyapunov regularization.
%\color{blue}
%The different learnable families act on the vector of raw spectral parameters
%$S_\phi = (S_{\phi,1},\ldots,S_{\phi,d})^\top$ through a variant-specific
%\emph{spectral shaping} map $h_\phi(\cdot)$ followed by a bounded squashing,
%so that each diagonal coefficient satisfies
%\[
%\Sigma_{\phi,i} \;=\; \rho_{\max}\,\sigma\!\big(h_\phi(S_{\phi,i})\big)
%\;\in\; (0,\rho_{\max}), \qquad i=1,\ldots,d,
%\]
%where $\sigma(\cdot)$ is the logistic sigmoid.
%\color{black}
\subsubsection{Spectral stability, contraction, and invertibility}

The spectral bound~\eqref{eq:rho_max_bound} directly implies stability and
contractivity of the associated linear dynamics.

\begin{proposition}[Spectral stability]
\label{prop:spectral_stability}
Let $\K_\phi$ be parameterized as in~\eqref{eq:koop_odo}, with orthonormal
$U_\phi, V_\phi$ and diagonal $\mathrm{diag}(\Sigma_\phi)$
satisfying~\eqref{eq:rho_max_bound}. Then
\begin{equation}
    \rho(\K_\phi)
    \;\le\;
    \|\K_\phi\|_2
    \;=\;
    \max_i \lvert \Sigma_{\phi,i} \rvert
    \;<\; \rho_{\max},
\end{equation}
where $\rho(\cdot)$ denotes the spectral radius and $\|\cdot\|_2$ the
spectral norm.
\end{proposition}

\begin{proof}
Because $U_\phi$ and $V_\phi$ are orthonormal, left and right multiplication
do not change the spectral norm:
\begin{equation}
    \|\K_\phi\|_2
    = \bigl\| U_\phi \, \mathrm{diag}(\Sigma_\phi) \, V_\phi^\top \bigr\|_2
    = \bigl\| \mathrm{diag}(\Sigma_\phi) \bigr\|_2.
\end{equation}
The spectral norm of a diagonal matrix is the maximum absolute diagonal
entry, hence
\begin{equation}
    \|\K_\phi\|_2 = \max_i \lvert \Sigma_{\phi,i} \rvert.
\end{equation}
For any matrix, the spectral radius is bounded above by the spectral norm,
$\rho(\K_\phi) \le \|\K_\phi\|_2$. Combining this with
$\lvert \Sigma_{\phi,i} \rvert < \rho_{\max}$ for all $i$ yields the claim.
\end{proof}

\begin{corollary}[Exponential contraction in latent space]
\label{cor:contraction}
Under the assumptions of Proposition~\ref{prop:spectral_stability}, for any
$\bm z_0 \in \mathbb{R}^d$ and any $n \in \mathbb{N}$,
\begin{equation}
    \| \K_\phi^n \bm z_0 \|_2
    \;\le\;
    \|\K_\phi\|_2^n \, \| z_0\|_2
    \;\le\;
    \rho_{\max}^n \, \| z_0\|_2,
\end{equation}
and hence $\lim_{n \to \infty} \K_\phi^n  z_0 =  0$.
\end{corollary}

\begin{proof}
Submultiplicativity of the spectral norm gives
$\|\K_\phi^n\|_2 \le \|\K_\phi\|_2^n$, so
\begin{equation}
    \|\K_\phi^n \bm z_0\|_2
    \le \|\K_\phi^n\|_2 \, \|\bm z_0\|_2
    \le \|\K_\phi\|_2^n \, \|\bm z_0\|_2.
\end{equation}
Proposition~\ref{prop:spectral_stability} implies
$\|\K_\phi\|_2 < \rho_{\max} < 1$, hence the right-hand side converges to
zero as $n \to \infty$, i.e. the latent dynamics \eqref{latentdyn} is exponentially stable:
\begin{equation}\label{cor_1}
  \|z_t\|_2
  \;\le\;
  \rho_{\max}^t \,\|z_0\|_2,
  \quad
  t \ge 0 .
\end{equation}
\end{proof}

In many forecasting settings, it is also natural to consider the inverse
dynamics whenever the spectrum is bounded away from zero. This leads to the
following simple consequence of the ODO structure.

\begin{proposition}[Invertibility and stability of the inverse]
\label{prop:inverse}
Assume the ODO parameterization~\eqref{eq:koop_odo} satisfies the two-sided
spectral bound
\begin{equation}
    0 < \rho_{\min} \le \lvert \Sigma_{\phi,i} \rvert \le \rho_{\max} < 1
    \quad\text{for all } i.
    \label{eq:two_sided_spectral_bound}
\end{equation}
Then $\K_\phi$ is invertible with
\begin{equation}
    \K_\phi^{-1}
    = V_\phi \, \mathrm{diag}\bigl(\Sigma_{\phi,1}^{-1},\dots,\Sigma_{\phi,d}^{-1}\bigr)
      \, U_\phi^\top,
\end{equation}
and
\begin{equation}
    \rho(\K_\phi^{-1})
    \;\le\;
    \|\K_\phi^{-1}\|_2
    \;=\;
    \max_i \lvert \Sigma_{\phi,i}^{-1} \rvert
    \;\le\; \rho_{\min}^{-1}.
\end{equation}
\end{proposition}

\begin{proof}
Under~\eqref{eq:two_sided_spectral_bound}, all diagonal entries of
$\mathrm{diag}(\Sigma_\phi)$ are nonzero, so it is invertible.
Since $U_\phi$ and $V_\phi$ are orthonormal, $\K_\phi$ is invertible and its
inverse has the stated form. The spectral norm calculation proceeds exactly
as in Proposition~\ref{prop:spectral_stability}, yielding
$\|\K_\phi^{-1}\|_2 = \max_i |\Sigma_{\phi,i}^{-1}|$, and
$\rho(\K_\phi^{-1}) \le \|\K_\phi^{-1}\|_2$.
\end{proof}

This shows that the same spectral factors that enforce forward-time contraction can also guarantee that the inverse dynamics are well-conditioned, provided the spectrum is kept away from zero.

\subsubsection{Expressiveness of spectrally bounded and low-rank operators}
The ODO factorization parameterizes exactly the class of linearly stable operators with bounded spectral norm.

\begin{proposition}[Representation of all spectrally bounded operators]
\label{prop:surjective}
Fix $\rho_{\max} > 0$ and define
\begin{equation}
    \mathcal{K}(\rho_{\max})
    :=
    \Bigl\{
       \K \in \mathbb{R}^{d\times d} \;:\;
       \|\K\|_2 \le \rho_{\max}
    \Bigr\}.
\end{equation}
Then the set of matrices that admit an ODO factorization with $\lvert \Sigma_{\phi,i} \rvert \le \rho_{\max}$ is precisely $\mathcal{K}(\rho_{\max})$. In particular, any $\K$ with $\|\K\|_2 < \rho_{\max}$ can be represented exactly by some choice of $(U_\phi,V_\phi,\Sigma_\phi)$.
\end{proposition}

\begin{proof}
(\emph{Surjectivity onto $\mathcal{K}(\rho_{\max})$.})
Let $\K \in \mathcal{K}(\rho_{\max})$. By the singular value decomposition,
there exist orthonormal matrices $U,V \in \mathbb{R}^{d\times d}$ and nonnegative
singular values $\Sigma_1,\dots,\Sigma_d$ such that
\begin{equation}
    \K = U \, \mathrm{diag}(\Sigma_1,\dots,\Sigma_d) \, V^\top.
\end{equation}
The spectral norm is $\|\K\|_2 = \max_i \Sigma_i$, so
$\|\K\|_2 \le \rho_{\max}$ implies $\Sigma_i \le \rho_{\max}$ for all $i$.
Thus $\K$ admits an ODO factorization with
$U_\phi = U$, $V_\phi = V$ and $\Sigma_{\phi,i} = \Sigma_i$ satisfying
$|\Sigma_{\phi,i}| \le \rho_{\max}$.

(\emph{Inclusion in $\mathcal{K}(\rho_{\max})$.})
Conversely, let
$\K_\phi = U_\phi \mathrm{diag}(\Sigma_\phi) V_\phi^\top$ with
$\max_i |\Sigma_{\phi,i}| \le \rho_{\max}$. As in the proof of
Proposition~\ref{prop:spectral_stability},
\begin{equation}
    \|\K_\phi\|_2
    = \bigl\|\mathrm{diag}(\Sigma_\phi)\bigr\|_2
    = \max_i |\Sigma_{\phi,i}|
    \le \rho_{\max},
\end{equation}
hence $\K_\phi \in \mathcal{K}(\rho_{\max})$.
\end{proof}

Thus the strictly stable and learnable Koopman variants with $\lvert \Sigma_{\phi,i} \rvert < \rho_{\max}$ retain the full expressiveness of all linear operators in the spectrally bounded class $\mathcal{K}(\rho_{\max})$, while providing an explicit stability margin.

\subsubsection{Low-rank structure and approximation}

For the low-rank family~\eqref{eq:koop_lowrank}, the Koopman operator is restricted
to a rank–$r$ subspace of $\mathbb{R}^d$.

\begin{proposition}[Low-rank structure and norm bound]
\label{prop:lowrank}
For the low-rank Koopman parameterization~\eqref{eq:koop_lowrank}:
\begin{enumerate}
    \item $\mathrm{rank}(\K_\phi) \le r$;
    \item $\|\K_\phi\|_2 = \max_{i \le r} \lvert \Sigma_{r,\phi,i} \rvert$;
    \item if $\lvert \Sigma_{r,\phi,i} \rvert < \rho_{\max}$ for all $i$,
          then $\rho(\K_\phi) \le \|\K_\phi\|_2 < \rho_{\max}$.
\end{enumerate}
\end{proposition}

\begin{proof}
(1) The product of a $d\times r$ matrix, an $r\times r$ diagonal matrix and
an $r\times d$ matrix has rank at most $r$. (2) Since $U_{r,\phi}$ and $V_{r,\phi}$ have orthonormal columns, they describe partial isometries from $\mathbb{R}^r$ into $\mathbb{R}^d$, and the nonzero singular values of $\K_\phi$ coincide with the absolute values of the
entries of $\Sigma_{r,\phi}$. Hence $\|\K_\phi\|_2 = \max_{i \le r} |\Sigma_{r,\phi,i}|$. (3) The last claim follows from (2) together with $\rho(\K_\phi) \le \|\K_\phi\|_2$ and the bound $|\Sigma_{r,\phi,i}| < \rho_{\max}$.
\end{proof}

Let $\K_\phi$ have singular values
$\Sigma_{\phi,1} \ge \dots \ge \Sigma_{\phi,d} \ge 0$, and let $\K_\phi^{(r)}$ denote the truncation of its
singular value decomposition to rank $r$, i.e., $\K_\phi^{(r)}$ retains only the first $r$ singular values
and vectors. The Eckart--Young theorem (see \cite{lin2011discrete} for more details) states that $\K_\phi^{(r)}$ minimizes
$\|\K_\phi - \tilde \K_\phi\|_F$ over all rank–$r$ matrices $\tilde \K_\phi$.
Any such $\K_\phi^{(r)}$ admits a representation of the form~\eqref{eq:koop_lowrank},
and thus belongs to the low-rank family. Consequently, whenever the latent
Koopman dynamics are approximately low-dimensional, the low-rank variant can
capture the principal modes while discarding high-rank noise.

\subsubsection{Lyapunov regularization and energy decay}

Besides spectral constraints, the training objective incorporates in \eqref{LyapEn} a
Lyapunov-inspired penalty that encourages contractive behavior with respect
to a fixed quadratic form. For latent pairs $(\bm z_t, \bm z_{t+1})$ the addition of the term
\begin{equation}
    \mathcal{L}_{\mathrm{Lyap}}
    ~=~
    \mathbb{E}
    \Bigl[
        \bigl(
          \|\bm z_{t+1}\|_P^2 - \|\bm z_t\|_P^2
        \bigr)_+
    \Bigr],
    \qquad
    \|\bm z\|_P^2 = \bm z^\top P \bm z,
\end{equation}
with $(x)_+ = \max\{x,0\}$ and $P \succ 0$, provides a common regularizer across learnable, Koopman variants, which further biases optimisation toward strict contraction of latent energy by penalizing unstable latent growth and promoting Koopman-consistent dynamics. The following results formalize these considerations.

\begin{proposition}[Lyapunov consistency]
\label{prop:lyap}
Suppose $P \succ 0$ and the support of $\bm z_t$ spans $\mathbb{R}^d$.
If $\mathcal{L}_{\mathrm{Lyap}} = 0$, then $\K_\phi$ satisfies the discrete
Lyapunov inequality
\begin{equation}
    \K_\phi^\top P \K_\phi - P \;\preceq\; 0.
    \label{eq:lyap_ineq}
\end{equation}
In particular, the quadratic energy $\|\bm z_t\|_P^2$ is non-increasing
under the dynamics $\bm z_{t+1} = \K_\phi \bm z_t$.
\end{proposition}

\begin{proof}
The equality $\mathcal{L}_{\mathrm{Lyap}} = 0$ implies
\begin{equation}
    \bigl(
      \|\bm z_{t+1}\|_P^2 - \|\bm z_t\|_P^2
    \bigr)_+
    = 0
\end{equation}
almost surely, and thus
\begin{equation}
    \|\bm z_{t+1}\|_P^2
    \le
    \|\bm z_t\|_P^2
\end{equation}
for all $\bm z_t$ in the support. Using $\bm z_{t+1} = \K_\phi \bm z_t$,
\begin{equation}
    \bm z_t^\top \K_\phi^\top P \K_\phi \bm z_t
    \le
    \bm z_t^\top P \bm z_t,
\end{equation}
or equivalently
\begin{equation}
    \bm z_t^\top (\K_\phi^\top P \K_\phi - P) \bm z_t \le 0
\end{equation}
for all $\bm z_t$ in a set that spans $\mathbb{R}^d$. A symmetric matrix
$M$ satisfies $M \preceq 0$ if and only if $\bm z^\top M \bm z \le 0$ for all
$\bm z$, hence $\K_\phi^\top P \K_\phi - P \preceq 0$.
\end{proof}

Conversely, classical results in linear systems theory relate such a
Lyapunov inequality to spectral stability.

\begin{proposition}[Spectral stability from a Lyapunov certificate]
\label{prop:lyap_spectral}
Let $\K \in \mathbb{R}^{d\times d}$ and suppose there exists $P \succ 0$ such
that
\begin{equation}
    \K^\top P \K - P \;\prec\; 0.
\end{equation}
Then $\rho(\K) < 1$.
\end{proposition}

\begin{proof}
Since $P \succ 0$, write $P = R^\top R$ for some invertible
$R \in \mathbb{R}^{d\times d}$ and consider the similarity transform
$\tilde \K = R \K R^{-1}$. The Lyapunov inequality becomes
\begin{equation}
    \K^\top P \K - P
    = \K^\top R^\top R \K - R^\top R
    = R^\top (\tilde \K^\top \tilde \K - I) R \prec 0.
\end{equation}
Premultiplying and postmultiplying by $(R^{-1})^\top$ and $R^{-1}$ yields
$\tilde \K^\top \tilde \K - I \prec 0$, which implies
$\|\tilde \K \bm x\|_2^2 < \|\bm x\|_2^2$ for all $\bm x \neq 0$, hence
$\|\tilde \K\|_2 < 1$. Therefore $\rho(\tilde \K) < 1$, and since $\K$ and
$\tilde \K$ are similar, $\rho(\K) = \rho(\tilde \K) < 1$.
\end{proof}

For finite Lyapunov regularization weight, the penalty $\mathcal{L}_{\mathrm{Lyap}}$ need
not vanish exactly; nevertheless,
Propositions~\ref{prop:lyap}--\ref{prop:lyap_spectral}
show that it biases the learned propagator toward the set of operators
admitting a Lyapunov certificate of stability, in a way that complements
the direct spectral constraints imposed by the ODO parameterization.

Taken together, Propositions~\ref{prop:spectral_stability}--\ref{prop:lyap_spectral}
show that the learnable Koopman family used in \ldkf: (i) provides a clear, tunable stability margin via $\rho_{\max}$ and the Lyapunov weight; (ii) retains the full expressiveness of spectrally bounded linear dynamics; (iii) admits low-rank specializations capturing dominant latent modes; and (iv) supports well-conditioned inverse propagation when the spectrum is bounded away from zero. These guarantees hold uniformly across all instantiations of \ldkf~built on top of the \patchtst, \autoformer, and \informer~backbones.

%%%%%%%%%%%%%%%%%%%%%%%%%%%%%%%%%%%%%%%%
%%%%%%%%%%%%%%%%%%%%%%%%%%%%%%%%%%%%%%%%
%%%%%%%%%%%%%%%%%%%%%%%%%%%%%%%%%%%%%%%%

\subsubsection{Theoretical stability advantage of the Koopman layer}
\label{subsec:koopman_stability_advantage}
We compare the latent dynamics induced by the proposed Koopman layer with those
of an unconstrained linear state--space model (SSM). We have shown how all
\ldkf~variants propagate a latent state $z_t \in \mathbb{R}^d$ according to
\begin{equation}\label{eq:koop_dynamics}
z_{t+1} = \K_\phi \, z_t, 
\end{equation}
where $\K_\phi$ is parameterised via an orthogonal--diagonal--orthogonal
(ODO) factorisation,
\begin{equation}
\K_\phi = U_\phi\,\mathrm{diag}\!\big(\Sigma_\phi \big)\,V_\phi^\top .
\end{equation}
Here $U_\phi, V_\phi \in \mathbb{R}^{d\times r}$ have orthonormal columns, obtained through QR retraction, and the spectral coefficients $\Sigma_\phi \in \mathbb{R}^r$ are generated from unconstrained raw parameters $S$ via a componentwise squashing map
\begin{equation}
  \Sigma_{i,\phi}(\theta)
  = \rho_{\max}\,\sigma\!\big(S_{i}\big),
  \quad
  \sigma:\mathbb{R}\to(0,1),
  \quad
  0 < \rho_{\max} < 1 .
  \label{eq:spectral_squashing}
\end{equation}

%In contrast, the SSM baseline evolves a hidden state
%$h_t \in \mathbb{R}^{d_h}$ as 
%\begin{equation}
%  h_{t+1} = A h_t + B x_t,
%  \qquad
%  y_T = C h_T ,
%  \label{eq:ssm_baseline}
%\end{equation}
%where $A,B,C$ are completely unconstrained.

In contrast, the SSM baseline evolves a hidden state $h_t \in \mathbb{R}^{d_h}$ as
\begin{equation}
h_{t+1} = A h_t + B x_t, \qquad \hat y_t = C h_t,
\label{eq:ssm_baseline}
\end{equation}
where $A,B,C$ are completely unconstrained.
%In the following proposition we discuss spectral bound and contraction of the Koopman layer.
%\begin{proposition}
%\label{prop:koopman_spectral_bound}
%Let $\K(\theta)$ be defined as above with $0<\rho_{\max}<1$. Then, for all
%$\theta$,
%\begin{equation}
%  \rho\!\big(\K(\theta)\big)
%  \;\le\;
%  \|\K(\theta)\|_2
%  \;=\;
%  \max_i |\Sigma_i(\theta)|
%  \;<\;
%  \rho_{\max},
%\end{equation}
%Beyond this hard spectral constraint, training includes a Lyapunov-inspired
%regulariser
%\begin{equation}
  %\mathcal{L}_{\mathrm{Lyap}}
%  =
  %\mathbb{E}\!\left[
    %\max\big\{0,\,
    %\|z_{t+1}\|_2^2 - \|z_t\|_2^2
    %\big\}
  %\right],
%\end{equation}
%which further biases optimisation toward strict contraction of latent energy.
In particular, no analogous spectral and contraction constraints to the ones for $\mathcal{K}_\phi$ are imposed on the SSM transition matrix $A$ in
\eqref{eq:ssm_baseline}. In the following Theorem we formally address core stability advantage of the Koopman parameterisation.

\begin{theorem}
\label{thm:core_stability_advantage}
Let $\mathcal{K}_\phi(\rho_{\max})$ denote the class of Koopman operators realised by
the ODO-based spectral--squashing parameterisation, such that
$\|\K_\phi\|_2 < \rho_{\max} < 1$ for all $\K_\phi \in \mathcal{K}_\phi({\rho_{\max}})$.
Let $\mathcal{A}$ denote the class of unconstrained state--space transition matrices. Then: (i) For every $\K_\phi \in \mathcal{K}_\phi({\rho_{\max}})$, the latent dynamics $z_{t+1} = \K_\phi z_t$ are uniformly exponentially stable at all training iterates. (ii) The class $\mathcal{A}$ contains matrices with arbitrarily large spectral radius; moreover, for any finite prediction horizon, there exist unstable $A \in \mathcal{A}$ that achieve the same finite-horizon training loss as a stable model.
\end{theorem}

\begin{proof}
\paragraph{(i) Stability of the Koopman parameterisation} 
%Proposition~\ref{prop:spectral_stability} establishes that $\|\K\|_2 < \rho_{\max} < 1$ uniformly over all training iterates. Hence \[
  %\|z_t\|_2
  %= \|\K_\phi^t z_0\|_2
  %\le \|\K_\phi\|_2^t \|z_0\|_2
  %\le \rho_{\max}^t \|z_0\|_2,\] which proves uniform exponential stability.
By Proposition~\ref{prop:spectral_stability} and Corollary~\ref{cor:contraction}, for all $\K_\phi \in \mathcal{K}_\phi(\rho_{\max})$ the latent dynamics satisfy (see \eqref{cor_1})
\begin{equation}
\|z_t\|_2 \le \rho_{\max}^t \|z_0\|_2,
\end{equation}
which proves uniform exponential stability.
\color{black}

\paragraph{(ii) Unstable but loss-equivalent SSM solutions}
The unconstrained class $\mathcal{A}$ contains matrices with arbitrarily large
spectral radius: for any $\gamma>0$, $A=\gamma I$ satisfies $\rho(A)=\gamma$.
Now consider a stable SSM
$h_{t+1}=Ah_t+Bx_t$, $\hat y_t = C h_t$
that achieves a given finite-horizon loss over $t=0,\dots,H$.
Define an augmented state
$\tilde h_t = [h_t^\top,\, w_t^\top]^\top$ with dynamics
\begin{equation}\label{htwt}
  \tilde h_{t+1}
  =
  \begin{bmatrix}
    A & 0\\
    0 & \gamma I
  \end{bmatrix}
  \tilde h_t
  +
  \begin{bmatrix}
    B\\
    0
  \end{bmatrix}
  x_t,
  \qquad
  \hat y_t =
  \begin{bmatrix}
    C & 0
  \end{bmatrix}
  \tilde h_t,
\end{equation}
where $\gamma>1$ and $\tilde h_0 = [h_0^\top,\,0^\top]^\top$.
Then $w_t \equiv 0$ for all $t\le H$, so the augmented model produces identical
outputs and achieves the same finite-horizon loss, while its transition matrix
has spectral radius $\gamma>1$. Choosing $\gamma$ arbitrarily large completes
the proof.
\end{proof}

%%%%%%%%%%%%%%%%%%%%%%%%%%%%%%%%%%%%%%%%
%%%%%%%%%%%%%%%%%%%%%%%%%%%%%%%%%%%%%%%%
%%%%%%%%%%%%%%%%%%%%%%%%%%%%%%%%%%%%%%%%

\subsubsection{Invertibility in finite dimensions and Banach space extension}

Let $(\mathcal{X}, \|\cdot\|)$ be a Banach space and let
$T:\mathcal{X}\to\mathcal{X}$ denote the (possibly nonlinear) time-$1$
evolution map of an underlying dynamical system.
The Koopman operator acting on observables $J:\mathcal{X}\to\mathbb{R}$ is
\begin{equation}
  (\mathcal{K}_\phi J)(x) \;=\; J\big(T(x)\big),
\end{equation}
which is linear even if $T$ is nonlinear. If $T$ is bijective, then $\mathcal{K}_\phi$ is invertible with
$\mathcal{K}_\phi^{-1}J = J\circ T^{-1}$. Thus invertibility of the latent propagator corresponds to bijectivity of the underlying state update. In \dkoopformer~we learn a \emph{finite-dimensional approximation} of
$\mathcal{\K}_\phi$ using an ODO-structured linear operator
\[
\K_\phi 
= U_\phi\,\mathrm{diag}(\Sigma_\phi)\,V_\phi^\top,
\quad
U_\phi,V_\phi\in O(d),\;\; \Sigma_i \in (0,\rho_{\max}) .
\]
%Because $0<\Sigma_{i,\phi}<\rho_{\max}<1$, the operator $\K_\phi$ belongs to the group
%$\mathrm{GL}(d)$ and satisfies \(\K_\phi^{-1} = V_\phi\,\mathrm{diag}(\Sigma_\phi^{-1})\,U^\top, \quad \|\K_\phi\|_2 \le \rho_{\max}, \quad \|\K_\phi^{-1}\|_2 \le 1/\sigma_{\min}(\K_\phi) < 1/\rho_{\max}. \)
Because $0 < \Sigma_{i,\phi} < \rho_{\max} < 1$, the operator $\K_\phi$ is invertible and belongs to the general linear group
\[
GL(d) := \{ M \in \mathbb{R}^{d \times d} \mid \det(M) \neq 0 \},
\]
that is, the group of all invertible $d \times d$ real matrices. Its inverse is given by
\[
\K_\phi^{-1} = V_\phi \, \mathrm{diag}(\Sigma_\phi^{-1}) \, U_\phi^\top,
\]
with operator norm bounds $\|K_\phi\|_2 \le \rho_{\max}$ and
$\|\K_\phi^{-1}\|_2 \le 1/\sigma_{\min}(\K_\phi) < 1/\rho_{\max}$. 

%This provides a finite-dimensional surrogate of an invertible Koopman semigroup: \( \K_\phi^n \to 0 \;\text{as}\; n\to\infty, \quad \K_\phi^{-n} \;\text{well-defined and bounded for all $n$}.\) In contrast, in a generic SSM with transition matrix $A$, invertibility and boundedness fail in general: $A$ may be singular, defective, or unstable, and even if $A$ is invertible, its condition number is uncontrolled. Thus the Koopman operator used in \ldkf~forms a stable, invertible, finite-dimensional representation of an underlying invertible dynamics, while an unconstrained SSM provides no such guarantees.
This yields a finite-dimensional surrogate of an invertible discrete-time Koopman evolution operator, for which $\K_\phi^n \to 0$ as $n \to \infty$ while backward iterates $K_\phi^{-n}$ remain well-defined and bounded.

In contrast, standard state-space models with unconstrained transition matrices $A$ (as commonly used in classical linear SSMs and recent deep state-space architectures) provide no such guarantees: $A$ may be singular, ill-conditioned, or unstable, and even when invertible its spectral properties are not explicitly controlled \cite{gu2021s4,smith2022simplified,hamilton2020time}.

Consequently, the Koopman propagator used in Learnable-DeepKoopFormer yields a stable, invertible, finite-dimensional representation of latent dynamics, whereas unconstrained SSM formulations lack explicit spectral and stability guarantees.

%%%%%%%%%%%%%%%%%%%%%%%%%%%%%%%%%%%%%%%%
%%%%%%%%%%%%%%%%%%%%%%%%%%%%%%%%%%%%%%%%
%%%%%%%%%%%%%%%%%%%%%%%%%%%%%%%%%%%%%%%%

\section{Numerical Simulations}\label{sec:simulations}

To evaluate the proposed \ldkf~architecture in a controlled yet systematically scalable setting, we employ a unified Koopformer benchmark pipeline on multivariate real-valued time series. From each dataset we construct sliding windows of length $P$ and corresponding forecast targets of length $H$, using a grid of input–output configurations with $P \in \{80, 100, 120, 140\}$ and $H \in \{4, 8, 12, 16\}$. The windows are normalised feature-wise and split into training and test sets using an $80\%/20\%$ chronological partition.

\paragraph{Model hyperparameters}
To ensure a fair comparison, we keep the capacity of all architectures fixed across experiments. For both \patchtst~and \informer, the Transformer encoder uses $L=3$ layers with hidden dimension $d_{\mathrm{model}} = 96$, $h = 4$ attention heads, and feed-forward width $d_{\mathrm{ff}} = 96$. The \autoformer~backbone employs the same number of layers and heads with $d_{\mathrm{model}} = 96$ and a slightly narrower feed-forward width $d_{\mathrm{ff}} = 64$. In all cases, the patch size (for \patchtst~and \informer) and the moving-average window (for \autoformer) are tied to the input window length $P$, and token-level representations are pooled into a single latent vector.

The Koopman layer in the constrained and learnable regimes acts on a latent state of dimension
$d_{\mathrm{lat}} = d_{\mathrm{model}}$ for \patchtst~and \informer, and $d_{\mathrm{lat}} = H$ for \autoformer. The strictly stable and learnable variants enforce an upper bound $\rho(\K_\phi) < \rho_{\max}$ with $\rho_{\max} = 0.99$, whereas the unconstrained variant uses a free dense matrix in $\mathbb{R}^{d_{\mathrm{lat}} \times d_{\mathrm{lat}}}$. Within the \ldkf~family, the  \ssdkf~and \pmcdkf~operators use global and dimension-wise $(\alpha,\beta)$ parameters, respectively; \mlpdkf~employs a one-hidden-layer MLP to parameterise the spectrum; and the \lrdkf~variant uses rank $r = 16$, i.e.\ $\K_\phi = U_{r,\phi} \operatorname{diag}(\Sigma_{r,\phi}) V_{r,\phi}^\top$ with $U_r, V_r \in \mathbb{R}^{d_{\mathrm{lat}} \times 16}$.

For the non-Koopman baselines, the LSTM forecaster uses $L=2$ recurrent layers with hidden size $96$, the DLinear model implements a single shared linear map from the length-$P$ history of each channel to the $H$-step forecast, and the discrete-time state-space model uses a latent state of dimension $96$ with a linear readout to $\mathbb{R}^{H \cdot d}$.  All models are trained with the Adam optimizer (learning rate $3 \times 10^{-4}$) for $4{,}000$ gradient steps per $(P,H)$ configuration; the Lyapunov regularisation weight is set to $\lambda_{\mathrm{Lyap}} = 0.1$ for constrained, learnable, and unconstrained Koopman variants unless stated otherwise.

\paragraph{Spectral logging and diagnostics}
For all models with an explicit linear propagator, we record spectral quantities that describe stability and latent energy dynamics. In Koopman–based architectures, the update
$z_{t+1} = \K_\phi z_t$ is controlled by the spectrum of the learned transition matrix $\K_\phi$, which determines how latent signals contract or amplify over time.

In the constrained and learnable ODO variants, the transition operator is written in factorised form
\(
  \K_\phi = U_\phi \Sigma_\phi V_\phi^\top,
\)
where $U_\phi$ and $V_\phi$ are orthonormal and $\Sigma_\phi = \mathrm{diag}(\sigma_{i,\phi})$ is produced by a squashing map restricting each $\sigma_{i,\phi}$ to $(0,\rho_{\max})$ with $\rho_{\max}<1$. We therefore log the singular values $\{\sigma_i(\K_\phi)\}$ and in particular track $\max_i \sigma_i(\K_\phi)$ as a proxy for the spectral radius. This provides a direct view of how different parametrisations ( \texttt{scalar, per–mode, MLP–based, low–rank}) populate the disc of radius $\rho_{\max}$.

To disentangle the effect of spectral structure in the latent dynamics, we evaluate three linear-transition regimes: \emph{constrained Koopman} (\texttt{constr}), \emph{unconstrained Koopman} (\texttt{unconstr}), and a standard \emph{linear state-space model (SSM)} baseline. In the constrained regime, the latent propagator is parameterized via an orthogonal--diagonal--orthogonal (ODO) factorization,
$\K_\phi = U_\phi \mathrm{diag}(\Sigma_\phi) V_\phi^\top$, where the spectral coefficients are squashed to satisfy
$\max_i |\Sigma_{\phi,i}| < \rho_{\max} < 1$, ensuring a uniform contraction bound and stable latent rollouts.
In the unconstrained Koopman regime, we retain the same encoder--propagator--decoder pipeline but replace the ODO structure with a free dense matrix
$\K_\phi \in \mathbb{R}^{d\times d}$, learned directly from data without explicit spectral-radius constraints; this variant maximizes linear expressiveness but may admit unstable or poorly conditioned transitions, hence we monitor its eigen-spectrum $\{\lambda_i(K_\phi)\}$ and spectral radius $\rho(\K_\phi)=\max_i |\lambda_i(\K_\phi)|$ during training.
As a non-Koopman baseline, we consider a discrete-time linear SSM of the form
$h_{t+1} = A h_t + B x_t$ with prediction $\hat{y}_t = C h_t$, where the transition matrix $A$ is likewise learned without explicit spectral control.
Because $A$ is generally non-normal, transient growth is governed by the operator norm rather than eigenvalues; accordingly, we log its singular values (in particular the maximum singular value of $A$) as a stability/conditioning diagnostic.
Together, these three regimes provide a controlled comparison between (i) explicitly spectrally shaped latent dynamics (constrained Koopman), (ii) fully free linear latent dynamics within the Koopman pipeline (unconstrained Koopman), and (iii) classical unconstrained linear state-space evolution (SSM).
\color{black}

%In contrast, the unconstrained Koopman \color{red}[THE only UNCONSTRAINED operators mentioned in the paper was the "linear SSM" in (34)!]\color{black} operators do not impose spectral limits, and the matrix $\K_\phi$ is learned freely. Here we log the full eigenvalue set $\{\lambda_i(\K_\phi)\}$ together with the spectral radius \(
  %\rho(\K_\phi) = \max_i |\lambda_i(\K_\phi)|, \) which reveals whether training pushes the dynamics toward marginal stability on the unit circle or toward more strongly contractive behaviours.

%Finally, for the linear state–space baseline \color{red}[IS it the same of (34)?WHY now B and C are missing?]\color{black} $h_{t+1} = Ah_t$ we record the singular values $\{\sigma_i(A)\}$ rather than eigenvalues, because $A$ is generally non-normal and transient amplification is governed by the operator norm $\|A\|_2$. The largest singular value controls worst–case growth via \( \|A^k\|_{2} \leq \sigma_{\max}(A)^k.\)

Overall, this spectral logging procedure converts each trained model into a spectral data point, enabling distributional comparisons (e.g., violin plots) that expose stability–accuracy trade–offs across Koopman parametrisations and backbones.

\paragraph{Benchmark Datasets}\label{sec:data_set}
To evaluate the forecasting performance of \ldkf~we investigate three heterogeneous and dynamically rich real-world domains: \emph{climate systems}, \emph{financial markets}, and \emph{electricity generation}. These domains exhibit nonstationarity, multiscale temporal dependencies, and nonlinear interactions—conditions under which classical deep forecasting models often face difficulties.

For climate analysis and renewable energy applications, we employ atmospheric data from CMIP6\footnote{\url{https://cds.climate.copernicus.eu/datasets/projections-cordex-domains-single-levels?tab=overview}} and ERA5\footnote{\url{https://cds.climate.copernicus.eu/datasets}}, capturing regional wind speed and surface pressure over Germany~\cite{makula2022coupled, hersbach2020era5, forootani_cadnn_2025}. As a complementary nonlinear system, we analyze cryptocurrency market behavior using a publicly available multivariate dataset\footnote{\url{https://github.com/Chisomnwa/Cryptocurrency-Data-Analysis}}, which includes volatility, pricing trends, and trading activity for major digital assets—an archetypal example of stochastic, regime-switching dynamics. Finally, we examine trends in national energy production by modeling electricity generation in Spain\footnote{\url{https://github.com/afshinfaramarzi/Energy-Demand-electricity-price-Forecasting/tree/main}} across fossil, wind, solar, hydro, and biomass sources, where physical constraints and weather dependencies induce complex temporal couplings. Together, these datasets span distinct dynamical regimes and provide a rigorous foundation for assessing the ability of Koopman-enhanced representations to extract latent linear structure from nonlinear temporal processes.

\paragraph{Hardware Configuration}
All experiments were executed on a high-performance computing (\texttt{HPC}) system with dual \texttt{AMD EPYC 9554} processors (128 cores per node, multithreaded) and 1.5\,TB shared memory, suitable for large context windows and attention-based models. Training was accelerated using \texttt{NVIDIA L40S GPUs} supporting mixed-precision tensor operations. Resource allocation and job scheduling were managed through the \texttt{Slurm} workload manager with exclusive node access to ensure reproducible performance.

\paragraph{Code availability statement}
The \ldkf~framework---including source code, datasets, and representative figures---is accessible at \texttt{Github} \footnote{\url{https://github.com/Ali-Forootani/Learnable-DeepKoopFormer}} and \texttt{Zenodo} \footnote{\url{ https://doi.org/10.5281/zenodo.17988424}, \url{https://doi.org/10.5281/zenodo.18115612}}. The entire implementation is written in \texttt{Python} and builds upon standard scientific computing libraries, including \texttt{PyTorch}, \texttt{NumPy}, and \texttt{SciPy}.

%%%%%%%%%%%%%%%%%%%%%%%%%%%%%%%%%%%%%%%%%%%%%%%%%
%%%%%%%%%%%%%%%%%%%%%%%%%%%%%%%%%%%%%%%%%%%%%%%%%
%%%%%%%%%%%%%%%%%%%%%%%%%%%%%%%%%%%%%%%%%%%%%%%%%

\subsection{Discussion on Wind Speed dataset}\label{sec:wind_speed}

For the CMIP6‐based wind–speed forecasting task,
Fig.~\ref{fig:wind_error} depicts the distributions of train/test MSE and MAE across all patch–length and horizon configurations. The Koopman–enhanced Transformers (\patchtst, \autoformer, \informer) exhibit \emph{tightly concentrated, low-centred} violins in all four panels, indicating consistently low errors and a remarkable stability across forecasting conditions. In contrast, the recurrent and linear baselines (LSTM, DLinear, SSM) display \emph{higher-centred} distributions with visibly broader shapes, reflecting larger error levels and increased sensitivity to patch length and horizon changes. Within the Koopman family, the \texttt{constr} variant shows the most compact and low-error violins, particularly for \patchtst, highlighting the effect of spectral constraints on stabilising latent dynamics. The learnable variants (\ssdkf, \pmcdkf, \mlpdkf, \lrdkf) maintain similarly low medians with slightly wider spreads, suggesting a trade-off between expressive dynamics and variability. The \texttt{unconstr} operator produces noticeably broader tails—especially in the \patchtst~ backbone—which implies that removing spectral control can reduce robustness, even though median errors remain competitive.

Overall, based on violin distributions alone, Koopman–Transformer models are not only more accurate than recurrent or linear baselines, but also substantially more stable across a wide range of forecasting configurations.

\paragraph{Koopman spectra across backbones - Wind Speed}
For each trained configuration (choice of patch length and forecast horizon), we extract the spectrum of the latent transition operator and pool all spectral magnitudes $\rho = |\lambda|$ at the level of \autoformer, \informer, and \patchtst\ backbones and their Koopman variants. Figure~\ref{fig:wind_spectra} reports these pooled distributions as violins, together with the singular-value spectrum of the SSM baseline. Across all backbones, the strictly stable Koopman variants (``\texttt{constr}'') exhibit compact spectra: the bulk of the mass lies between approximately $\rho \approx 0.3$ and $\rho \approx 0.8$, with essentially no values above~$1$, confirming that the spectral-radius constraint is respected in practice. The learnable families ( \ssdkf, \pmcdkf, \mlpdkf, \lrdkf) occupy a similar central range but display visibly broader violins and slightly heavier upper tails, indicating that these parametrisations exploit a larger portion of the admissible disc while still remaining predominantly contractive. In contrast, the unconstrained Koopman variants (``\texttt{unconstr}'') concentrate their spectra very close to the origin, with most magnitudes below $\rho \approx 0.2$, which reflects the combined effect of data fitting and the Lyapunov regularisation term driving the latent dynamics towards strongly damped behaviour. The SSM baseline forms a narrow distribution at small spectral magnitudes (around $\rho \approx 0.05$–$0.2$) with a thin tail extending towards and \emph{beyond the unit circle}, consistent with occasional large singular values of the state-transition matrix in an unconstrained linear state-space model. Overall, the figure shows that the dominant differences in spectral behaviour are induced by the Koopman parametrisation (constrained vs.\ learnable vs.\ unconstrained) rather than by the specific Transformer backbone, and that the Koopman-enhanced models achieve expressive yet predominantly stable latent linear dynamics.

\begin{figure*}
    \centering
    \includegraphics[width=0.99\linewidth]{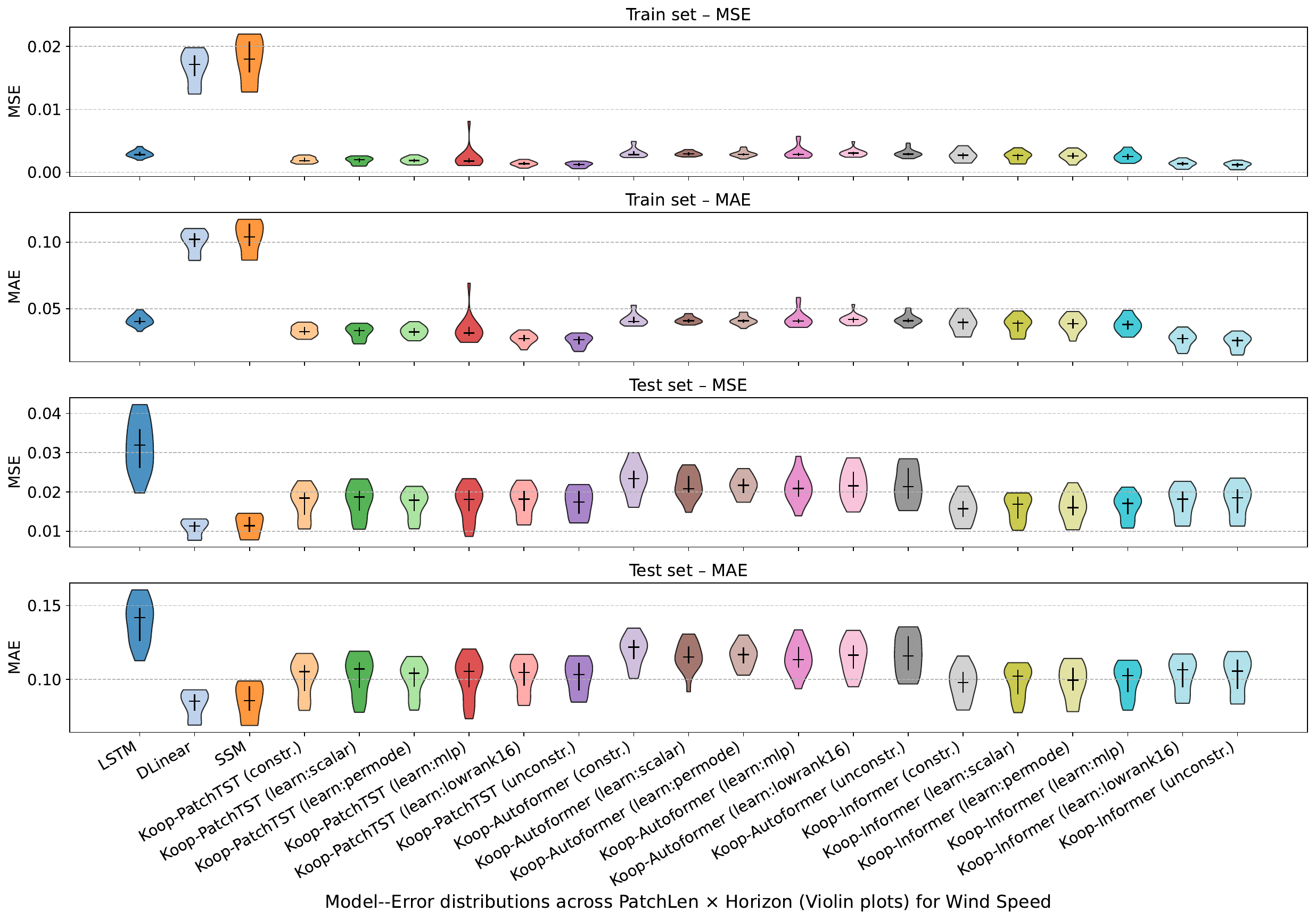}
    \caption{Violin plots of train/test MSE and MAE for all architectures on
    the wind speed forecasting task, aggregated over all patch lengths and
    horizons.}
    \label{fig:wind_error}
\end{figure*}

\begin{figure}
    \centering
    \includegraphics[width=1.01\linewidth]{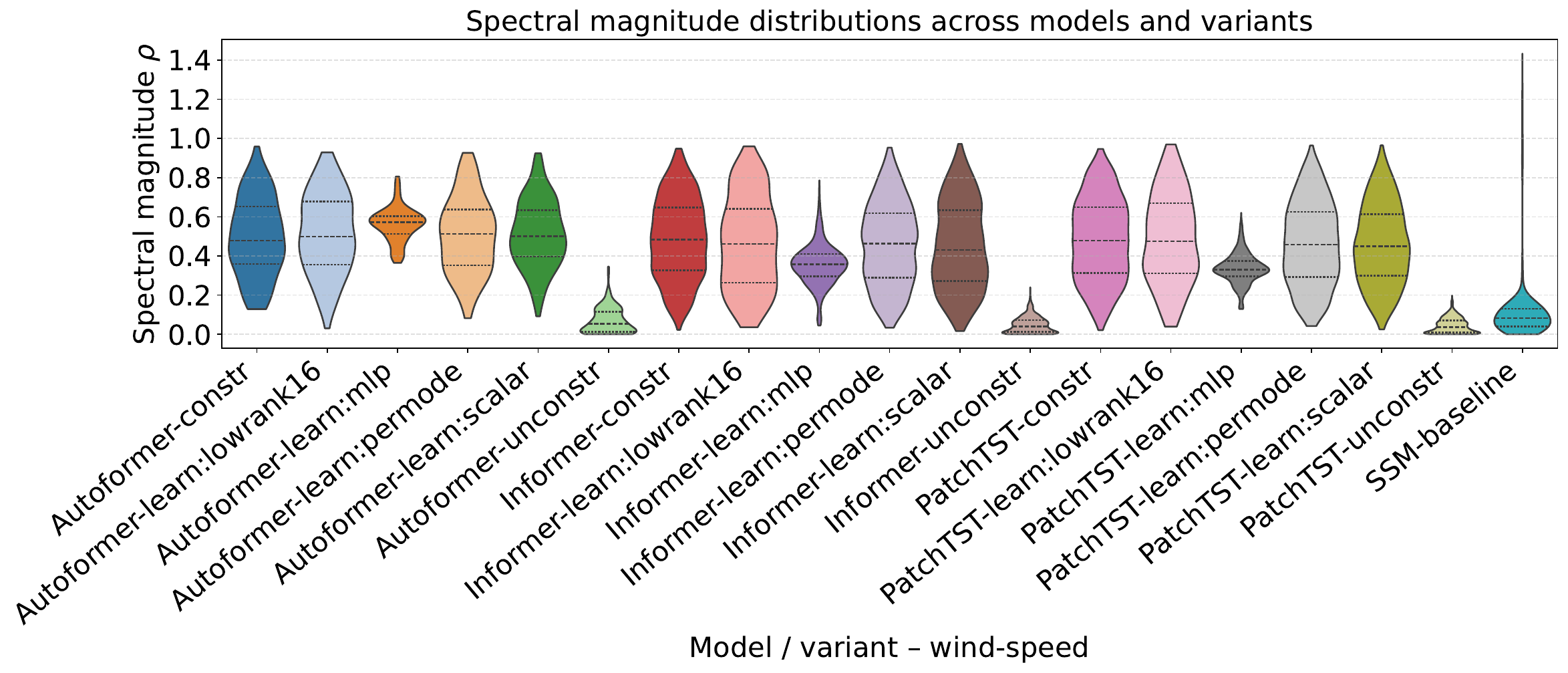}
    \caption{Spectral distributions of latent Koopman operators for CIMIP6 wind speed forecasting task.}
    \label{fig:wind_spectra}
\end{figure}

\subsection{Discussion on Pressure Surface dataset}\label{sec:pressure_surface}

For the CMIP6-based pressure–surface forecasting task, 
Fig.~\ref{fig:pressure_error} summarises the error distributions over all
patch–length and horizon combinations via violin plots of train/test MSE and
MAE. Across the four violin panels, the Koopman–Transformer families
(\patchtst, \autoformer, \informer) concentrate most of their mass in the lower
part of the error axis, with visibly tight and compact distributions. In
contrast, the purely recurrent and linear baselines (LSTM, DLinear, SSM) exhibit
higher-centred violins with noticeably wider spread in several cases. Within the Koopman–enhanced models, the constrained operator (\texttt{constr}) yields the most compact and low-variance distributions across all backbones. The learnable operator families (\ssdkf, \pmcdkf, \mlpdkf, \lrdkf) produce slightly broader violins, reflecting increased expressive
capacity with a small variability trade-off. By contrast, the unconstrained operator (\texttt{unconstr}) exhibits heavier tails—especially for \patchtst{} and \informer—even when its median error remains competitive, indicating reduced robustness when no spectral regularisation is enforced. A notable exception emerges in the test–set summary: while Koopman variants remain among the most stable across patch lengths and horizons, \emph{certain
baseline configurations} (particularly DLinear and SSM) achieve \emph{slightly lower mean test errors} for the pressure–surface task. However, these gains are offset by substantially wider violin spreads, indicating less predictable generalisation across forecasting conditions. In contrast, the Koopman family achieves comparable or near-best means with significantly reduced spread, suggesting a more stable latent dynamics model that balances expressive powerand robustness rather than relying on chance favourable configurations.

\paragraph{Koopman spectra across backbones - Pressure Surface}

For each trained configuration on the pressure–surface data, we extract all
latent spectral magnitudes $\rho = |\lambda|$ and pool them by backbone and
Koopman variant. Figure~\ref{fig:pressure_spectra} shows the resulting
distributions. Across \autoformer, \informer, and \patchtst, the constrained variants
(\texttt{constr}) concentrate their spectra in a compact band
$\rho \approx 0.3$–$0.8$, with no values exceeding~$1$, indicating that the
spectral–radius constraint is enforced in practice. The learnable variants
(\ssdkf, \pmcdkf, \mlpdkf, \lrdkf) occupy a similar range but with broader
distributions and mildly heavier tails, reflecting increased expressive
capacity while remaining largely contractive. The unconstrained variants (\texttt{unconstr}) collapse toward strongly damped dynamics, with most magnitudes below $\rho \approx 0.2$. The SSM baseline likewise concentrates at small values but exhibits a thin upper tail extending
towards and \emph{occasionally above the unit circle}. Overall, differences in spectral
behavior arise predominantly from the Koopman parameterisation rather than from
the choice of Transformer backbone, and Koopman–enhanced models realise stable
latent dynamics for large–scale pressure–surface forecasting.

\begin{figure*}
    \centering
    \includegraphics[width=0.99\linewidth]{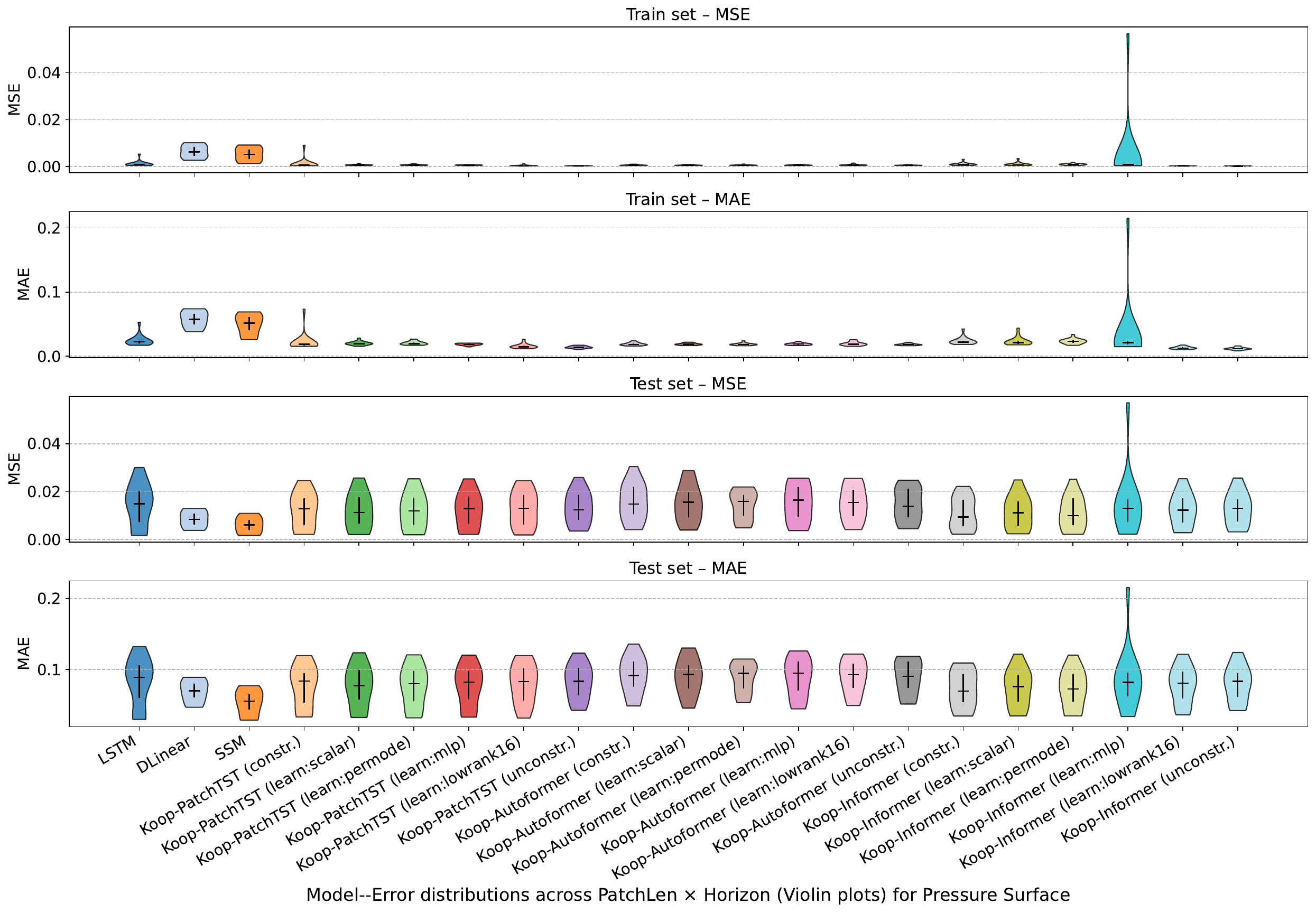}
    \caption{Violin plots of train/test MSE and MAE for all architectures on
    the pressure surface forecasting task, aggregated over all patch lengths and
    horizons.}
    \label{fig:pressure_error}
\end{figure*}

\begin{figure}
    \centering
    \includegraphics[width=1.01\linewidth]{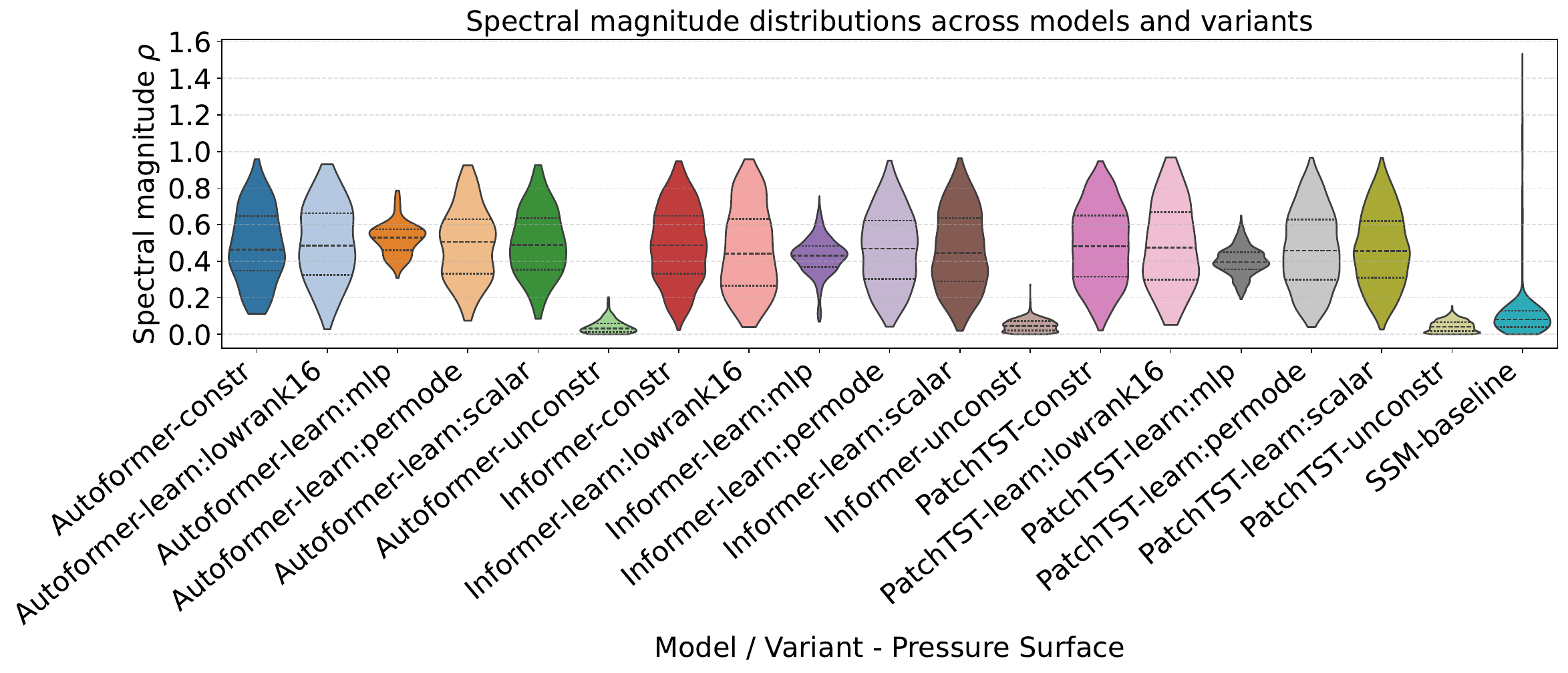}
    \caption{Spectral distributions of latent Koopman operators for CIMIP6 pressure surface forecasting task.}
    \label{fig:pressure_spectra}
\end{figure}

\subsection{Discussion on Cryptocurrency dataset}\label{sec:crypto_forecast}

For the cryptocurrency forecasting task, Fig.~\ref{fig:crypto_error} reports the distribution of train/test MSE and MAE across all patch–length and horizon configurations using violin plots. The Koopman–Transformer families (\patchtst, \autoformer, \informer) generally form compact, low-centred distributions, indicating stable forecasting behaviour across hyperparameter choices. In contrast, the recurrent and linear baselines (LSTM, DLinear, SSM) display visibly broader spreads and higher variability, suggesting more sensitive performance to architectural and hyperparameter changes. Among Koopman variants, the constrained operator (\texttt{constr}) achieves the most concentrated error distributions on both training and test sets, highlighting the benefits of explicit spectral stabilisation in highly fluctuating crypto time series. Learnable operator families (\ssdkf, \pmcdkf, \mlpdkf, \lrdkf) obtain competitive or near-best median errors, albeit with slightly wider spreads that reflect the trade-off between flexibility and stability. By contrast, the unconstrained operator (\texttt{unconstr}) exhibits heavier tails for \patchtst{} and \informer, despite a competitive median value, illustrating reduced robustness when no spectral control is imposed.

%Overall, the cryptocurrency results indicate that the Koopman-augmented models, particularly with spectral constraints, provide a stable forecasting mechanism under volatile dynamics. In comparison, baseline architectures occasionally match median performance but exhibit higher sensitivity to patch–length and horizon variation, resulting in less predictable generalisation.

\paragraph{Koopman spectra across backbones — Cryptocurrency}

For the cryptocurrency forecasting experiments, we collect all latent spectral magnitudes $\rho = |\lambda|$ from the trained models and group them by backbone and Koopman variant. Figure~\ref{fig:crypto_spectra} displays the resulting distributions across \patchtst, \autoformer, and \informer. Across all three backbones, the constrained Koopman operator (\texttt{constr}) yields compact spectra centred around $\rho \approx 0.3$–$0.6$, with no values approaching the unit circle. This consistent boundedness indicates that the spectral constraint is effectively enforced, promoting stable latent evolution even under the pronounced volatility of crypto price dynamics. The learnable operator families (\ssdkf, \pmcdkf, \mlpdkf, \lrdkf) occupy similar or slightly wider ranges, typically $\rho \approx 0.3$–$0.8$, and exhibit moderately heavier tails. This reflects their increased expressive capacity: they can adaptively shape latent temporal scales while remaining largely contractive in practice. In contrast, the unconstrained variant (\texttt{unconstr}) collapses toward strongly damped spectra with $\rho < 0.2$, producing overly dissipative latent dynamics. Meanwhile, the SSM baseline concentrates at small magnitudes but shows a distinct thin tail extending toward and beyond $\rho \approx 1.0$, indicating occasional unstable latent modes in the absence of explicit spectral control.

\begin{figure*}
    \centering
    \includegraphics[width=0.99\linewidth]{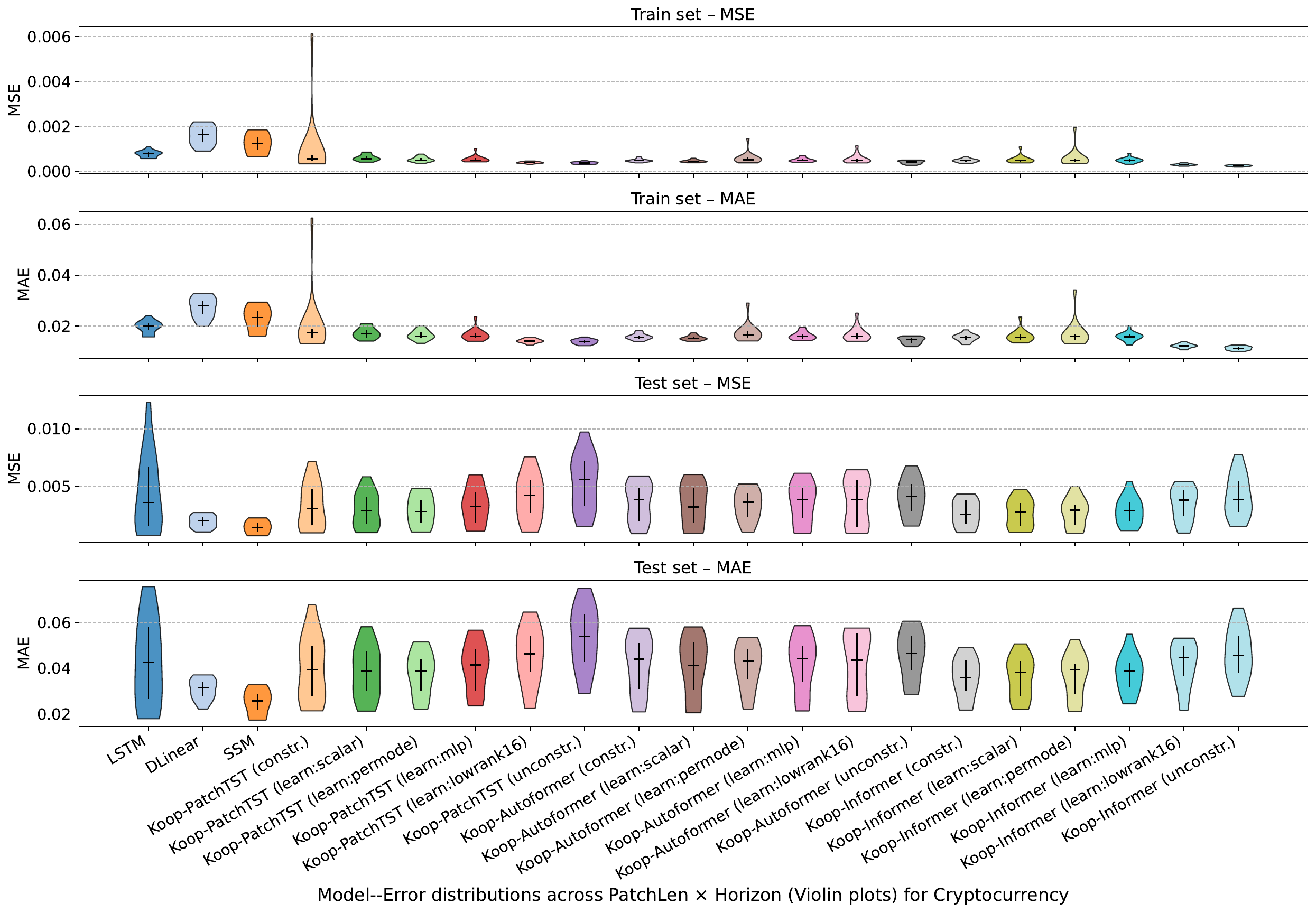}
    \caption{Violin plots of train/test MSE and MAE for all architectures on
    the Cryptocurrency forecasting task, aggregated over all patch lengths and
    horizons.}
    \label{fig:crypto_error}
\end{figure*}

\begin{figure}
    \centering
    \includegraphics[width=1.02\linewidth]{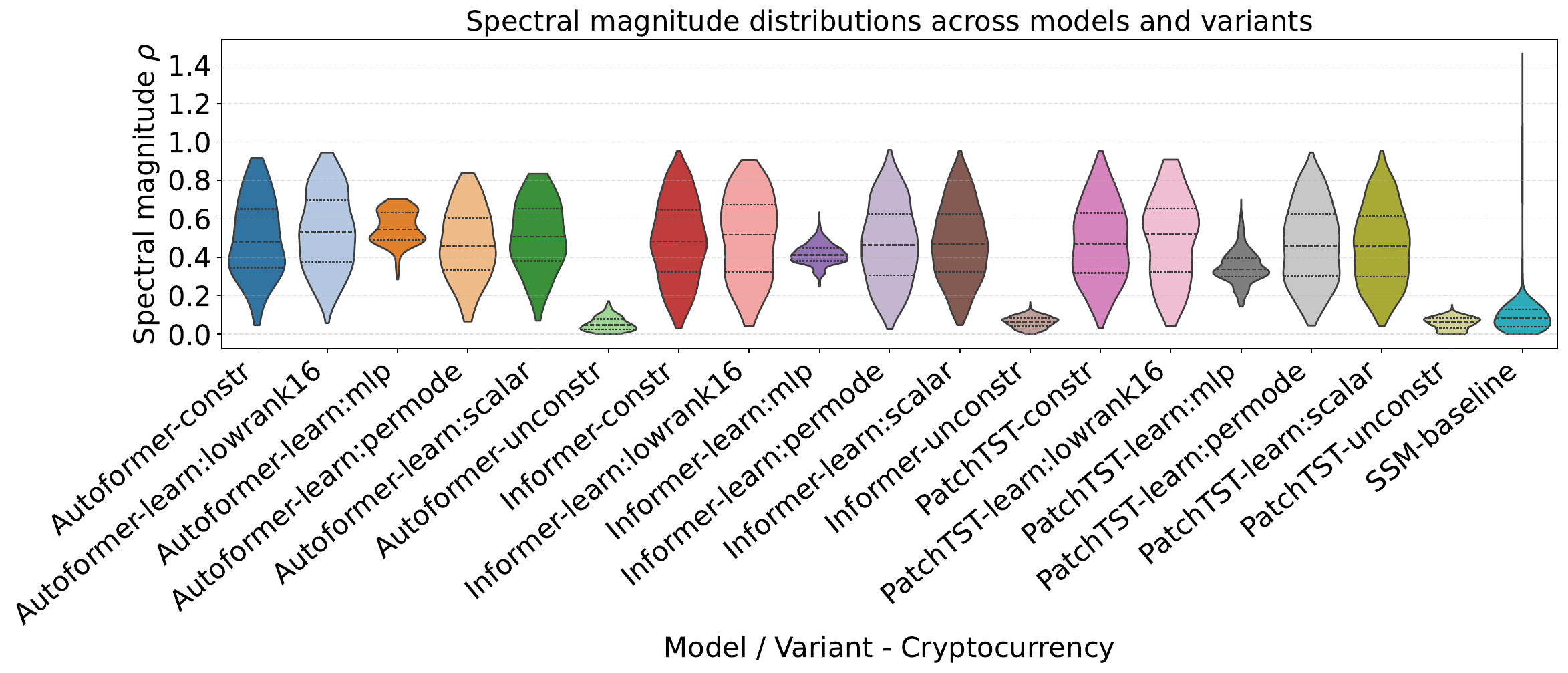}
    \caption{Spectral distributions of latent Koopman operators for the Cryptocurrency forecasting task.}
    \label{fig:crypto_spectra}
\end{figure}

\subsection{Discussion on Energy Systems Dataset}
\label{sec:energy_forecast}
Figure~\ref{fig:energy_error} shows the distribution of train/test MSE and MAE
over all patch--length and horizon settings for the energy systems task.
Unlike the previous domains, the \autoformer{} backbone performs noticeably
worse: both its baseline and Koopman variants produce high-centred violins with
broad tails, indicating weaker forecasting accuracy and a strong sensitivity to
hyperparameter choices. The LSTM baseline achieves very low training errors with compact violins, showing strong short-term adaptability. However, its test errors spread widely,
revealing limited generalisation and a tendency to overfit. Similarly, DLinear
and SSM occasionally attain competitive medians but exhibit inconsistent spreads
across configurations, suggesting dependence on specific window and horizon
settings rather than robust behaviour. Across all backbones, Koopman-enhanced variants yield more reliable performance. In particular, the constrained operator (\texttt{constr}) consistently produces tight, mid-centred violins with comparatively low variability. While not always the best in terms of median error, it is the most stable across hyperparameters,
highlighting the benefit of enforcing spectral bounds in the latent dynamics for slow-varying energy system trends. Learnable operators (\ssdkf, \pmcdkf, \mlpdkf, \lrdkf) achieve competitive medians with slightly broader spreads, reflecting a trade-off between expressiveness and robustness. Overall, the energy systems results demonstrate that minimising training error alone (e.g., LSTM or \autoformer) does not guarantee reliable forecasting. Koopman-enhanced models, particularly with spectral constraints, prioritise
stable generalisation across hyperparameter choices, making them well suited for
forecasting tasks governed by long-range temporal structure.

\paragraph{Koopman spectra across backbones — Energy Systems}

Figure~\ref{fig:energy_spectra} shows latent spectral magnitudes
$\rho = |\lambda|$ for all model variants in a single violin plot. The
constrained Koopman operator (\texttt{constr}) consistently produces compact
spectra around $\rho \approx 0.3$--$0.6$, clearly away from $\rho=1$, indicating
stable latent evolution suitable for slowly varying energy trends. The learnable
operator families (\ssdkf, \pmcdkf, \mlpdkf, \lrdkf) span slightly wider ranges
($\rho \approx 0.3$--$0.8$), reflecting greater flexibility while remaining
largely contractive. In contrast, the unconstrained variant
(\texttt{unconstr}) collapses toward very small radii ($\rho < 0.2$), implying
overly damped dynamics, whereas the SSM baseline shows predominantly small
values with a thin tail beyond $\rho \geq 1$, revealing occasional unstable
modes. Overall, explicit spectral control enables stable temporal modelling in
energy systems data.

\begin{figure*}
    \centering
    \includegraphics[width=0.99\linewidth]{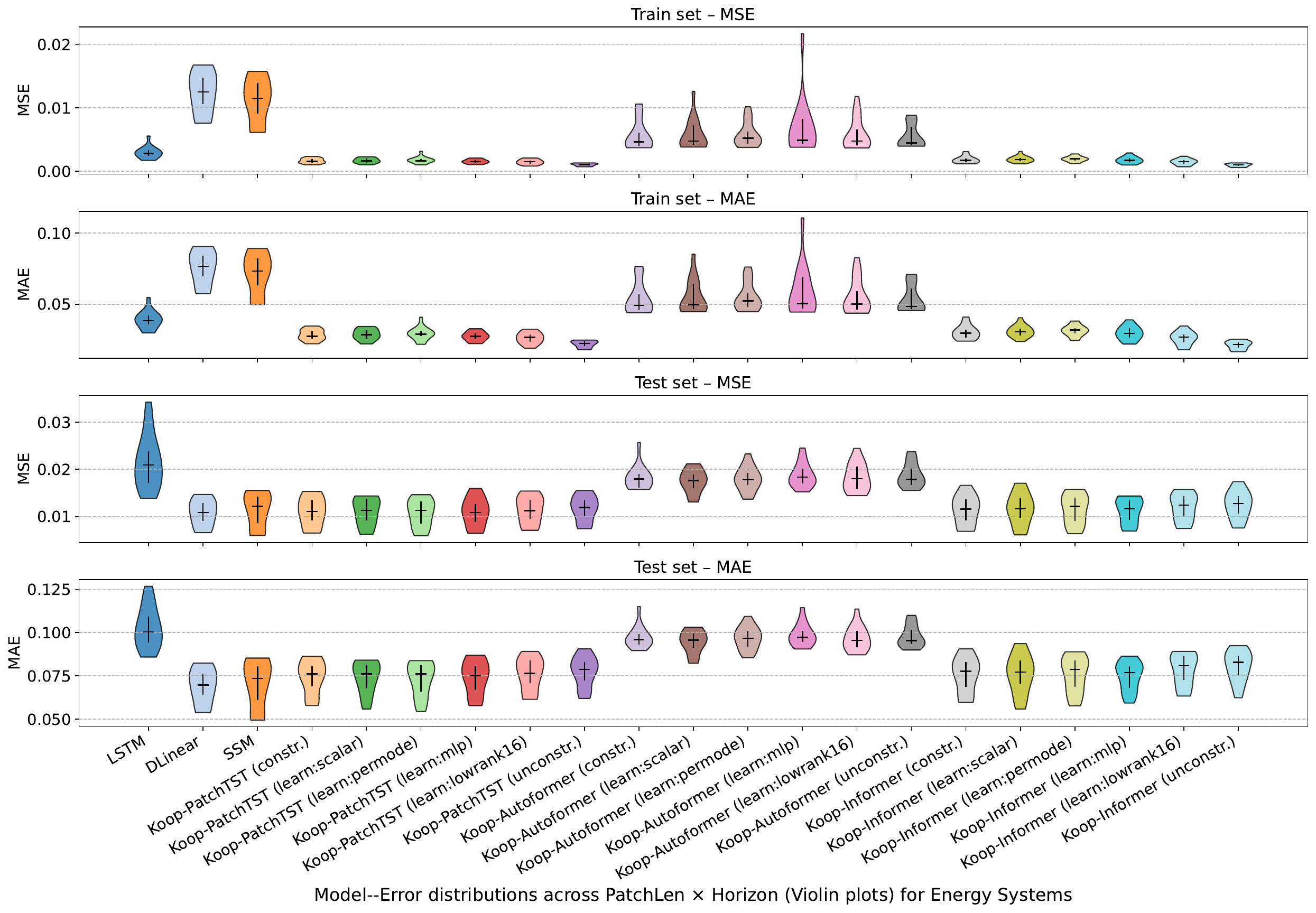}
   \caption{Violin plots of train/test MSE and MAE for all architectures on
    the energy systems forecasting task, aggregated over all patch lengths and
    horizons.}
    \label{fig:energy_error}
\end{figure*}

\begin{figure}
    \centering
    \includegraphics[width=1.02\linewidth]{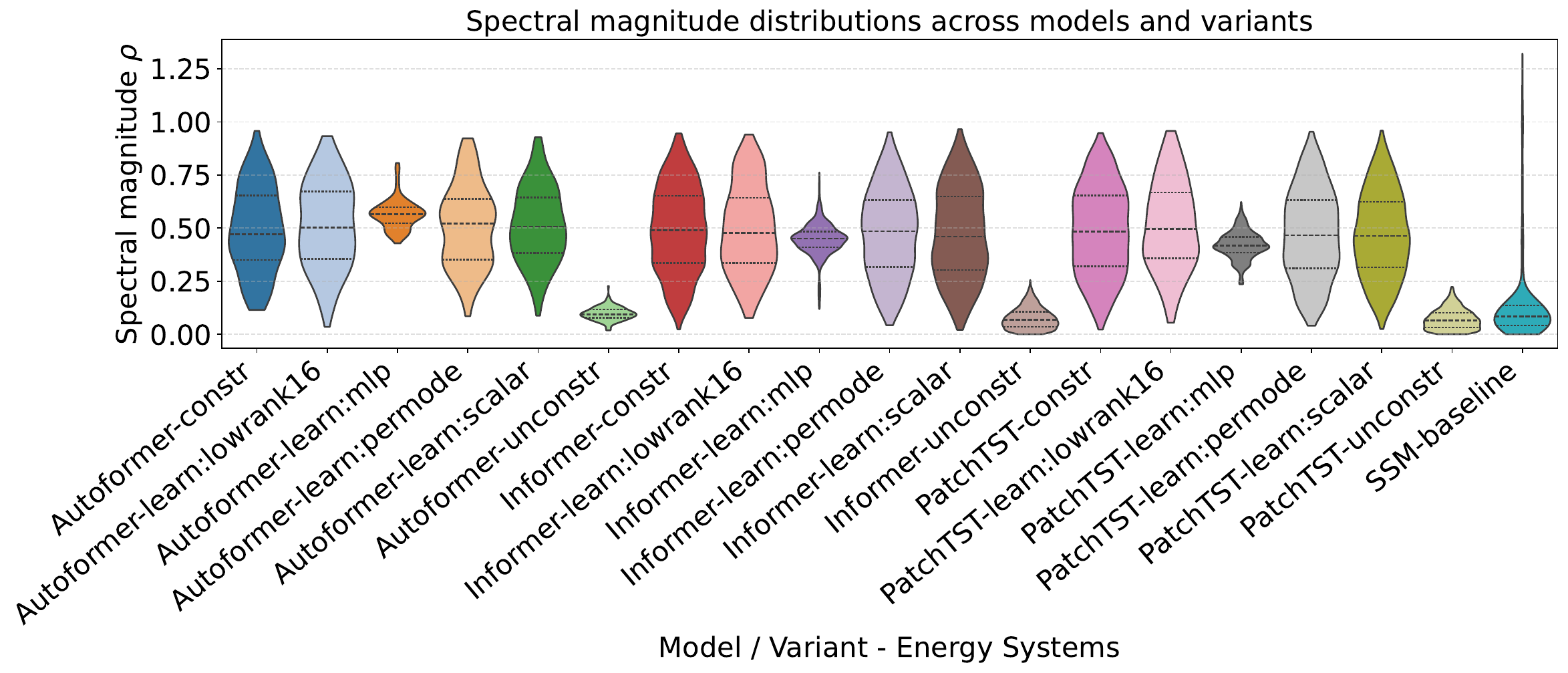}
    \caption{Spectral distributions of latent Koopman operators for the energy systems  forecasting task.}
    \label{fig:energy_spectra}
\end{figure}

\subsection{Discussion on ERA5 Wind Speed Dataset}
\label{sec:era5_forecast}

Figure~\ref{fig:era5_error} reports train/test error distributions for ERA5
wind forecasting across all patch--length and horizon settings. While most
methods reach low medians due to smooth wind dynamics, their robustness differs
substantially. LSTM, DLinear, and SSM produce competitive medians but wider distributions,
with LSTM displaying compact training violins yet broader test spreads,
indicating mild overfitting. DLinear and SSM fluctuate strongly across
settings, limiting reliability. By contrast, Koopman--enhanced models form consistently tight, low-centred violins. The constrained operator (\texttt{constr}) exhibits the most compact train/test distributions, demonstrating stable latent dynamics. Learnable
variants (\ssdkf, \pmcdkf, \mlpdkf, \lrdkf) achieve similar medians with
slightly broader spreads, balancing adaptivity and spectral stability. 

%These results show that robust wind forecasting benefits more from stable latent evolution than aggressive curve fitting, with spectrally regularised Koopman designs offering the most reliable performance.

\paragraph{Koopman spectra across backbones - ERA5 Wind Speed}
Figure~\ref{fig:era5_spectra} shows latent spectral magnitudes $\rho = |\lambda|$ for all model variants in a single violin plot. The constrained Koopman operator (\texttt{constr}) consistently produces compact spectra around $\rho \approx 0.3$--$0.6$, clearly away from $\rho=1$, indicating stable latent evolution suitable for slowly varying energy trends. The learnable operator families (\ssdkf, \pmcdkf, \mlpdkf, \lrdkf) span slightly wider ranges
($\rho \approx 0.3$--$0.8$), reflecting greater flexibility while remaining
largely contractive. In contrast, the unconstrained variant
(\texttt{unconstr}) collapses toward very small radii ($\rho < 0.2$), implying
overly damped dynamics, whereas the SSM baseline shows predominantly small
values with a thin tail beyond $\rho \geq 1$, revealing occasional unstable
modes. Overall, explicit spectral control enables stable temporal modelling in
energy systems data.

%\paragraph{Koopman spectra — ERA5 Wind Speed}
%Figure~\ref{fig:era5_spectra} shows that the constrained Koopman operator (\texttt{constr}) produces compact, mid-range spectra ($\rho \approx 0.3$--$0.6$), indicating stable latent evolution. Learnable variants extend to higher radii (up to $\rho \approx 0.9$), offering more flexibility while remaining largely contractive. The unconstrained operator collapses to $\rho < 0.2$, implying over-damped dynamics, while the SSM baseline concentrates at small magnitudes with a thin unstable tail ($\rho > 1$). Thus, spectral control yields stable and predictable latent dynamics for ERA5 wind forecasting.

\begin{figure*}
    \centering
    \includegraphics[width=0.99\linewidth]{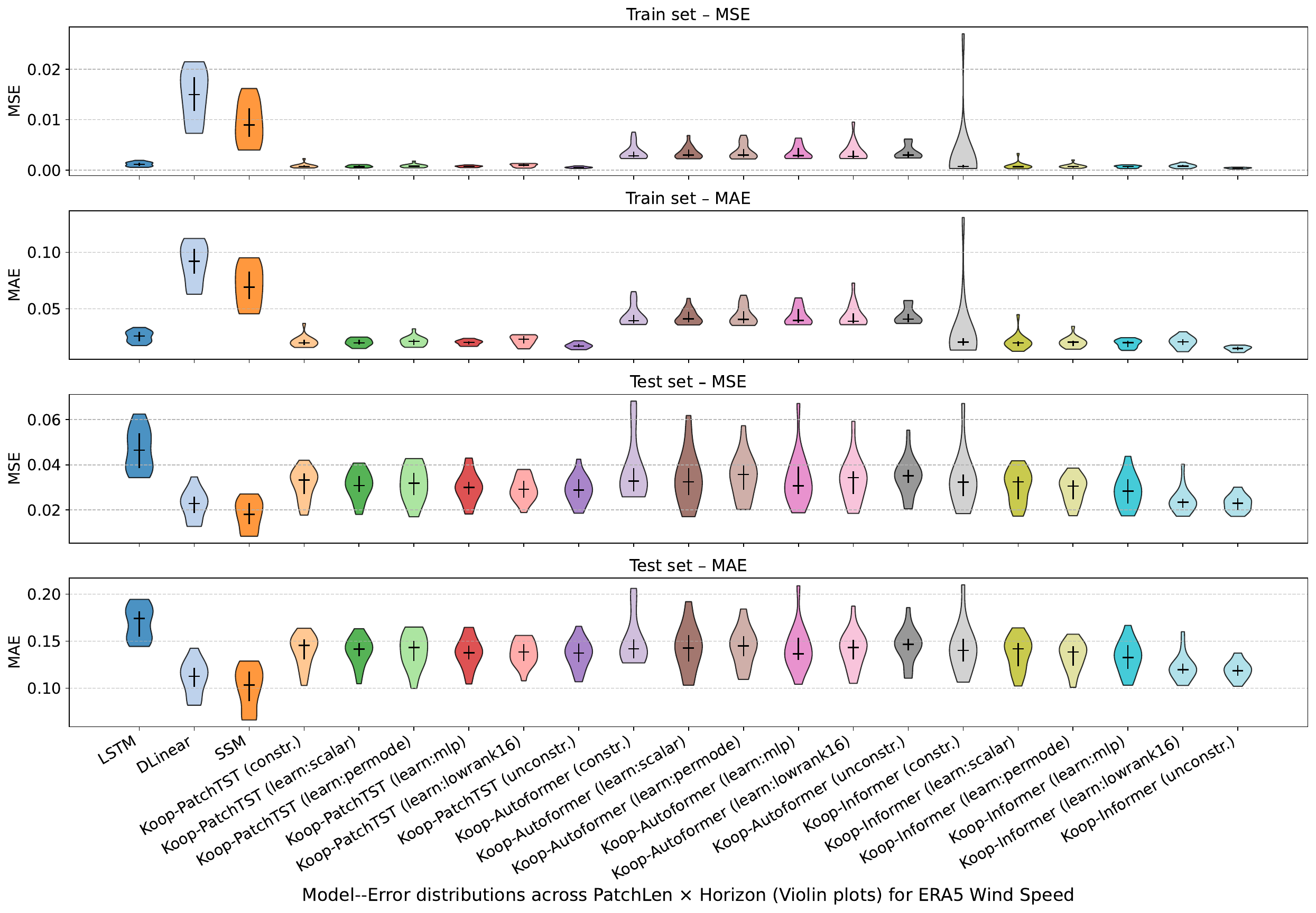}
   \caption{Violin plots of train/test MSE and MAE for all architectures on
    the ERA5 wind speed forecasting task, aggregated over all patch lengths and
    horizons.}
    \label{fig:era5_error}
\end{figure*}

\begin{figure}
    \centering
    \includegraphics[width=1.02\linewidth]{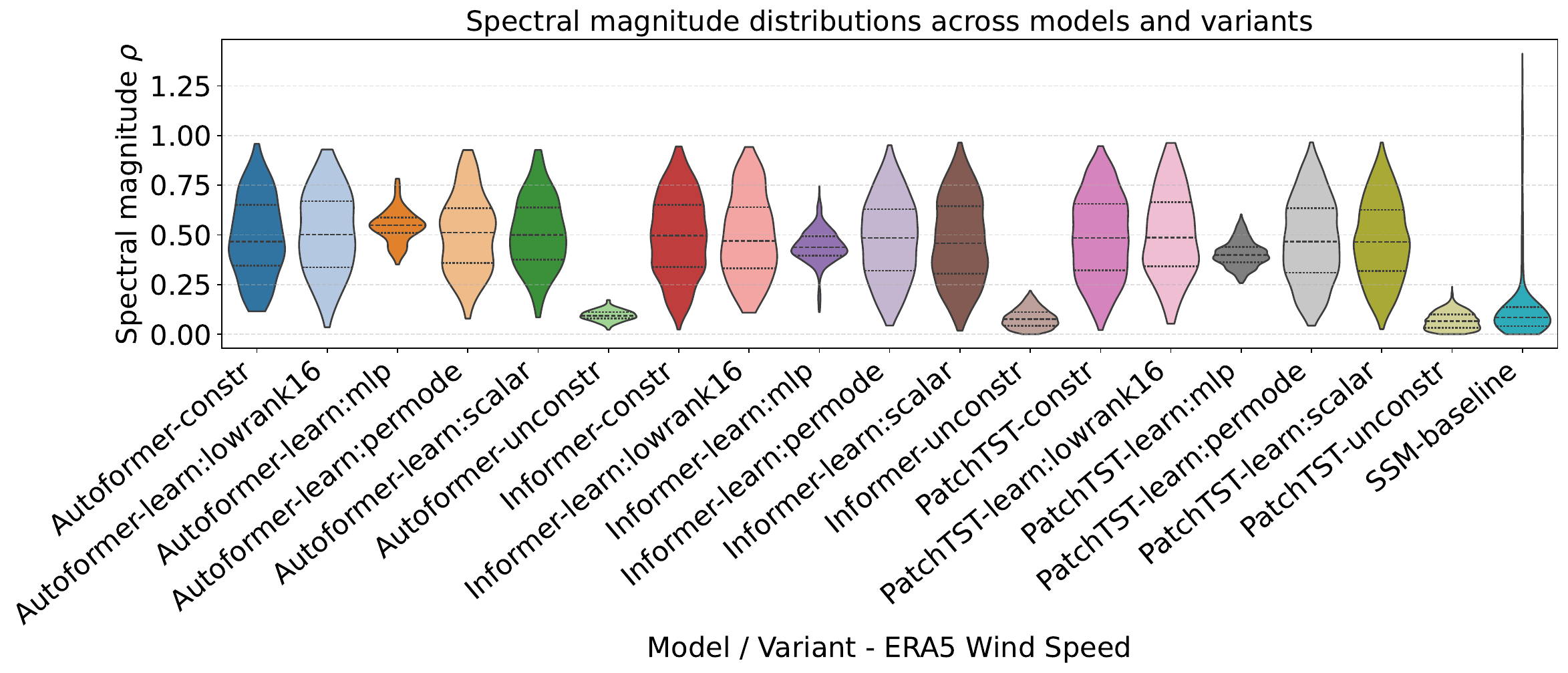}
    \caption{Spectral distributions of latent Koopman operators for ERA5 wind speed forecasting task.}
    \label{fig:era5_spectra}
\end{figure}

%%%%%%%%%%%%%%%%%%%%%%%%%%%%%%%%%%%%%%%%%%%%%%%%%%%%
%%%%%%%%%%%%%%%%%%%%%%%%%%%%%%%%%%%%%%%%%%%%%%%%%%%%

\subsection{Advantages of \ldkf~Modules.}
The spectral analysis highlights a key benefit of the proposed learnable Koopman 
modules: they adapt the latent time scales without sacrificing stable and 
reversible dynamics. Whereas purely constrained variants tightly restrict the 
spectrum, and unconstrained models tend to collapse into strongly damped states 
($|\lambda| \rightarrow 0$), the learnable families populate a non-collapsed 
but contractive range ($0.3 \lesssim |\lambda| \lesssim 0.8$). This region 
provides enough variability to capture both fast and slow temporal modes, 
while keeping the spectral radius below unity so that the latent dynamics 
remain stable. Crucially, because eigenvalues do not decay toward zero, the 
resulting operators are \emph{invertible}, enabling well-posed backward 
propagation and preventing information loss over long-range iterative 
forecasts. Thus, learnability does not merely increase flexibility; it yields 
\emph{invertible and non-degenerate latent dynamics}, offering richer temporal 
structure than constrained models and more reliable long-horizon behaviour 
than unstable or over-damped baselines.

%%%%%%%%%%%%%%%%%%%%%%%%%%%%%%%%%%%%%%%%%%%%%%%%%%
%%%%%%%%%%%%%%%%%%%%%%%%%%%%%%%%%%%%%%%%%%%%%%%%%%
%%%%%%%%%%%%%%%%%%%%%%%%%%%%%%%%%%%%%%%%%%%%%%%%%%

\section{Conclusion}\label{sec:conclusion}

This paper introduced the \ldkf~ framework, which augments lightweight and full Transformer forecasters with a family of spectrally controlled, learnable Koopman operators. Building on an orthogonal--diagonal--orthogonal (ODO) factorisation, we proposed scalar--gated, per--mode, MLP--shaped, and low--rank Koopman variants that interpolate between strictly constrained and fully unconstrained linear latent dynamics while preserving explicit control over spectrum, stability, and rank. We provided a unified training objective combining standard forecasting loss with Lyapunov--style regularisation, and established theoretical guarantees on spectral stability, contractivity, low--rank structure, and invertibility of the resulting propagators. Together, these results show that Koopman layers can be integrated into modern attention-based backbones in a way that is both mathematically principled and practically compatible with large-scale time-series forecasting.

\vspace{0.75em}
Extensive experiments across five heterogeneous real-world datasets---CMIP6 and ERA5 wind speed, CMIP6 surface pressure, cryptocurrency markets, and national-scale electricity generation---demonstrate that Koopman-enhanced Transformers achieve robust and accurate forecasting over a wide grid of input--output configurations. Across all domains, the Koopman variants embedded in \patchtst, \autoformer, and \informer~generally yield lower and more tightly concentrated train/test error distributions than recurrent (LSTM) and linear (\dlinear, SSM) baselines, indicating reduced sensitivity to patch length and horizon choices. Constrained Koopman operators systematically provide the most compact error violins and the most stable generalisation behaviour, while the learnable families (\ssdkf, \pmcdkf, \mlpdkf, \lrdkf) attain comparable or near-best median errors with modestly increased spread, reflecting a favourable trade-off between expressiveness and robustness. Unconstrained operators can match median accuracy but exhibit heavier tails, underscoring that the absence of spectral control tends to degrade reliability rather than improve it.

\vspace{0.75em}
The spectral diagnostics further clarify the role of the Koopman layer in shaping latent dynamics. Constrained variants consistently populate a mid-range contractive band ($0.3 \lesssim |\lambda| \lesssim 0.8$), ensuring non-explosive but non-collapsed evolution. Learnable variants occupy a similar range with slightly broader support, indicating that they exploit the admissible stability region to adjust latent time scales while remaining predominantly contractive. In contrast, unconstrained Koopman maps often collapse toward very small radii (over-damped dynamics), whereas the SSM baseline exhibits thin but non-negligible tails beyond the unit circle, confirming that unrestricted state-space models can drift into spectrally unstable regimes despite good finite-horizon fit. Overall, the spectral analysis shows that the primary differences in latent behaviour are induced by the Koopman parametrisation rather than by the particular Transformer backbone, and that \ldkf~achieves stable, interpretable, and invertible latent dynamics across diverse dynamical regimes.

\vspace{0.75em}
From a modelling perspective, the learnable Koopman modules offer a key structural advantage: they adapt latent temporal scales \emph{without sacrificing stability or invertibility}. By avoiding eigenvalues that cluster near zero or exceed the unit circle, the proposed operators maintain non-degenerate, well-conditioned dynamics that support long-range forecasting without catastrophic decay or uncontrolled growth. This makes \ldkf~particularly suitable for applications where both predictive accuracy and dynamical interpretability matter, such as climate risk assessment, energy systems planning, and financial stress testing, and suggests that spectral control is a powerful inductive bias for deep sequence models.

\vspace{0.75em}
Several directions remain open for future work. On the theoretical side, extending the analysis to explicitly stochastic or non-stationary settings, and connecting spectral constraints to distributional robustness and uncertainty quantification, are promising avenues. On the modelling side, integrating control inputs and exogenous covariates into the Koopman layer would enable joint forecasting and control for cyber--physical systems, while exploring richer backbones (e.g., state--space models with Koopman regularisation, graph-based or spatio--temporal Transformers) could further broaden applicability. Finally, scaling \ldkf~to larger benchmarks and online or continual-learning scenarios would help assess how learnable Koopman operators behave under persistent distribution shift. We view this work as a step toward a principled synthesis of operator-theoretic dynamics and modern deep forecasting, and anticipate that Koopman--enhanced architectures will become a useful building block in the broader toolbox of interpretable, stable sequence models.

%%%%%%%%%%%%%%%%%%%%%%%%%%%%%%%%%%%%%%%%%%%%%%%%%%%%%%%
%%%%%%%%%%%%%%%%%%%%%%%%%%%%%%%%%%%%%%%%%%%%%%%%%%%%%%%
%%%%%%%%%%%%%%%%%%%%%%%%%%%%%%%%%%%%%%%%%%%%%%%%%%%%%%%
%%%%%%%%%%%%%%%%%%%%%%%%%%%%%%%%%%%%%%%%%%%%%%%%%%%%%%%

\appendix

\begin{algorithm}
%\scriptsize
\caption{\dkoopformer~for Multivariate Time Series Forecasting}
\label{alg:koopformer}
\text{Input:} Multivariate time series $\{ x_t\}_{t=1}^T$, context length $P$, forecast horizon $H$\\
\text{Output:} Predicted sequence $\hat{ Y}_t = [\hat{x}_{t+P}, \dots, \hat{x}_{t+P+H-1}]$

\begin{algorithmic}[1]
\State \text{Data Preparation:}
    \begin{itemize}
        \item For each valid $t$, construct:
            \begin{align*}
                 X_t &= [ x_t, \dots, x_{t+P-1}] \in \mathbb{R}^{P \times d}, \\
                 Y_t &= [ x_{t+P}, \dots,  x_{t+P+H-1}] \in \mathbb{R}^{H \times d}.
            \end{align*}
    \end{itemize}

\State \text{Latent Encoding:}
    \begin{itemize}
        \item Compute latent vector: $ z_t = \mathcal{E}_\theta( X_t)$, where $\mathcal{E}_\theta$ is a Transformer-based encoder with positional encoding.
    \end{itemize}

\State \text{Koopman Operator Evolution:}
    \begin{itemize}
        \item Propagate the latent state: \(z_{t+1} =  \K_\phi  z_t.\)
        \item $\K_\phi$ is parameterized as: \(\K_\phi =  U_\phi\,\mathrm{diag}( \Sigma_\phi)\, V_\phi^\top,\)
            where $U_\phi, V_\phi$ are orthogonal matrices, $ \Sigma_\phi = \sigmoid( S)\, \rho_{\max}$, where $\sigma(\cdot)$ is the sigmoid function, $S$ raw parameters, and $\rho_{\max} < 1$.
    \end{itemize}

\State \text{Output Decoding (Direct $H$-Step Forecast):}
    \begin{itemize}
        \item Map the refined latent state to the $H$-step forecast via a linear decoder:
            \(
            \hat{ y}_t = \mathcal{D}_\varphi( z_{t+1}) \in \mathbb{R}^{H \cdot d},
            \)
            where $\mathcal{D}_\varphi$ is implemented as a fully connected  linear layer.
        \item Reshape to obtain the forecast sequence:
            \begin{multline*}
               \hat{Y}_t = \mathrm{reshape}\big(\hat{ y}_t,\, H,\, d\big)
            = [\hat{x}_{t+P}, \dots, \hat{x}_{t+P+H-1}] \\ \in \mathbb{R}^{H \times d}. 
            \end{multline*}
    \end{itemize}
\color{black}

\State \text{Loss Function:}
    \begin{align*}
        \mathcal{L} = \|\hat{Y}_t - Y_t\|^2 + \lambda \cdot \mathrm{ReLU}\left(\| z_{t+1}\|^2 - \|z_t\|^2\right).
    \end{align*}

\State \text{Optimization:} Train parameters $(\theta, \phi, \K_\phi)$ end-to-end with Adam.

\State \text{Inference:} At test time, given a context window $X_t$, compute
    \(
    z_t = \mathcal{E}_\theta( X_t), 
    \quad  z_{t+1} =  \K_\phi z_t,
    \)
    and obtain the $H$-step forecast in a single forward pass as
    \[
    \hat{ y}_t = \mathcal{D}_\varphi( z_{t+1}), 
    \qquad
    \hat{ Y}_t = \mathrm{reshape}(\hat{y}_t, H, d).
    \]

\end{algorithmic}
\end{algorithm}

%%%%%%%%%%%%%%%%%%%%%%%%%%%%%%%%%%%%%%%%%%%%%%%%%%%%%%%
%%%%%%%%%%%%%%%%%%%%%%%%%%%%%%%%%%%%%%%%%%%%%%%%%%%%%%%
%%%%%%%%%%%%%%%%%%%%%%%%%%%%%%%%%%%%%%%%%%%%%%%%%%%%%%%
%%%%%%%%%%%%%%%%%%%%%%%%%%%%%%%%%%%%%%%%%%%%%%%%%%%%%%%

\bibliographystyle{ieeetr}

\bibliography{refs/tutorial_bibliography,refs/koopman_ref_tutorial,refs/ref_koopformer, refs/references_brunton, refs/my_ref}

@inproceedings{NIPS2017_3a835d32,
 author = {Takeishi, Naoya and Kawahara, Yoshinobu and Yairi, Takehisa},
 booktitle = {Advances in Neural Information Processing Systems},
 editor = {I. Guyon and U. Von Luxburg and S. Bengio and H. Wallach and R. Fergus and S. Vishwanathan and R. Garnett},
 pages = {},
 publisher = {Curran Associates, Inc.},
 title = {Learning {Koopman} Invariant Subspaces for Dynamic Mode Decomposition},
 volume = {30},
 year = {2017}
}

@INPROCEEDINGS{Skomski2021,
  author={Skomski, Elliott and Vasisht, Soumya and Wight, Colby and Tuor, Aaron and Drgoňa, Ján and Vrabie, Draguna},
  booktitle={2021 American Control Conference (ACC)}, 
  title={Constrained Block Nonlinear Neural Dynamical Models}, 
  year={2021},
  volume={},
  number={},
  pages={3993-4000},
  doi={10.23919/ACC50511.2021.9482930}}

@ARTICLE{Drgona2022,
  author={Drgona, Jan and Tuor, Aaron and Vasisht, Soumya and Vrabie, Draguna},
  journal={IEEE Open Journal of Control Systems}, 
  title={Dissipative Deep Neural Dynamical Systems}, 
  year={2022},
  volume={1},
  number={},
  pages={100-112},
  doi={10.1109/OJCSYS.2022.3186838}}

@INPROCEEDINGS{Yeung2019,
  author={Yeung, Enoch and Kundu, Soumya and Hodas, Nathan},
  booktitle={2019 American Control Conference (ACC)}, 
  title={Learning Deep Neural Network Representations for {Koopman} Operators of Nonlinear Dynamical Systems}, 
  year={2019},
  volume={},
  number={},
  pages={4832-4839},
  doi={10.23919/ACC.2019.8815339}}

@inproceedings{NEURIPS2021_c9dd73f5,
 author = {Drgona, Jan and Mukherjee, Sayak and Zhang, Jiaxin and Liu, Frank and Halappanavar, Mahantesh},
 booktitle = {Advances in Neural Information Processing Systems},
 editor = {M. Ranzato and A. Beygelzimer and Y. Dauphin and P.S. Liang and J. Wortman Vaughan},
 pages = {24033--24047},
 publisher = {Curran Associates, Inc.},
 title = {On the Stochastic Stability of Deep Markov Models},
 volume = {34},
 year = {2021}
}

@article{Lusch2018,
	title        = {Deep Learning for Universal Linear Embeddings of Nonlinear Dynamics},
	author       = {Lusch, Bethany and Kutz, J. Nathan and Brunton, Steven L.},
	year         = 2018,
	month        = nov,
	journal      = {Nature Communications},
	publisher    = {{Nature Publishing Group}},
	volume       = 9,
	number       = 1,
	pages        = 4950,
	issn         = {2041-1723},
	copyright    = {2018 The Author(s)},
	ids          = {Lusch2018a},
}

@article{PIML2021,
title = "Physics-informed machine learning",
author = "Karniadakis, {George Em} and Kevrekidis, {Ioannis G.} and Lu Lu and Paris Perdikaris and Sifan Wang and Liu Yang",
note = "Publisher Copyright: {\textcopyright} 2021, Springer Nature Limited.",
year = "2021",
month = jun,
doi = "10.1038/s42254-021-00314-5",
language = "English (US)",
volume = "3",
pages = "422--440",
journal = "Nature Reviews Physics",
issn = "2522-5820",
publisher = "Springer International Publishing",
number = "6",
}

@INPROCEEDINGS{Truong2023,
  author={Nghiem, Truong X. and Drgoňa, Ján and Jones, Colin and Nagy, Zoltan and Schwan, Roland and Dey, Biswadip and Chakrabarty, Ankush and Di Cairano, Stefano and Paulson, Joel A. and Carron, Andrea and Zeilinger, Melanie N. and Shaw Cortez, Wenceslao and Vrabie, Draguna L.},
  booktitle={2023 American Control Conference (ACC)}, 
  title={Physics-Informed Machine Learning for Modeling and Control of Dynamical Systems}, 
  year={2023},
  volume={},
  number={},
  pages={3735-3750},
  doi={10.23919/ACC55779.2023.10155901}}

@article{forootani2025deepkoopformer,
  title={DeepKoopFormer: A Koopman Enhanced Transformer Based Architecture for Time Series Forecasting},
  author={Forootani, Ali and Khosravi, Mohammad and Barati, Masoud},
  journal={arXiv preprint arXiv:2508.02616},
  year={2025}
}

@article{lin2011discrete,
  title={Discrete Eckart--Young theorem for integer matrices},
  author={Lin, Matthew M},
  journal={SIAM journal on matrix analysis and applications},
  volume={32},
  number={4},
  pages={1367--1382},
  year={2011},
  publisher={SIAM}
}

@article{dlinear,
  title={Are Transformers Effective for Time Series Forecasting?},
  author={Ailing Zeng and Muxi Chen and Lei Zhang and Qiang Xu},
  journal={arXiv preprint arXiv:2205.13504},
  year={2022}
}

@inproceedings{autoformer,
  title={Autoformer: Decomposition Transformers with {Auto-Correlation} for Long-Term Series Forecasting},
  author={Haixu Wu and Jiehui Xu and Jianmin Wang and Mingsheng Long},
  booktitle={Advances in Neural Information Processing Systems},
  year={2021}
}

@inproceedings{fedformer,
  title={{FEDformer}: Frequency enhanced decomposed transformer for long-term series forecasting},
  author={Zhou, Tian and Ma, Ziqing and Wen, Qingsong and Wang, Xue and Sun, Liang and Jin, Rong},
  booktitle={Proc. 39th International Conference on Machine Learning},
  location = {Baltimore, Maryland},
  pages={},
  year={2022}
}

@inproceedings{pyraformer,
title={Pyraformer: Low-Complexity Pyramidal Attention for Long-Range Time Series Modeling and Forecasting},
author={Liu, Shizhan and Yu, Hang and Liao, Cong and Li, Jianguo and Lin, Weiyao and Liu, Alex X and Dustdar, Schahram},
booktitle={International Conference on Learning Representations},
year={2022}
}

@inproceedings{multichannel,
  title={Time series classification using multi-channels deep convolutional neural networks},
  author={Zheng, Yi and Liu, Qi and Chen, Enhong and Ge, Yong and Zhao, J Leon},
  booktitle={International conference on web-age information management},
  pages={298--310},
  year={2014},
  organization={Springer}
}

@article{makula2022coupled,
  title={Coupled Model Intercomparison Project phase 6 evaluation and projection of East African precipitation},
  author={Makula, Exavery Kisesa and Zhou, Botao},
  journal={International Journal of Climatology},
  volume={42},
  number={4},
  pages={2398--2412},
  year={2022},
  publisher={Wiley Online Library}
}

@article{khosravi2023representer,
  title={Representer theorem for learning Koopman operators},
  author={Khosravi, Mohammad},
  journal={IEEE Transactions on Automatic Control},
  volume={68},
  number={5},
  pages={2995--3010},
  year={2023},
  publisher={IEEE}
}

@article{hersbach2020era5,
  title={The ERA5 global reanalysis},
  author={Hersbach, Hans and Bell, Bill and Berrisford, Paul and Hirahara, Shoji and Hor{\'a}nyi, Andr{\'a}s and Mu{\~n}oz-Sabater, Joaqu{\'\i}n and Nicolas, Julien and Peubey, Carole and Radu, Raluca and Schepers, Dinand and others},
  journal={Quarterly journal of the royal meteorological society},
  volume={146},
  number={730},
  pages={1999--2049},
  year={2020},
  publisher={Wiley Online Library}
}

@inproceedings{zeng2023transformers,
  title={Are transformers effective for time series forecasting?},
  author={Zeng, Ailing and Chen, Muxi and Zhang, Lei and Xu, Qiang},
  booktitle={Proceedings of the AAAI conference on artificial intelligence},
  volume={37},
  number={9},
  pages={11121--11128},
  year={2023}
}

@article{forootani_cadnn_2025,
title = {Climate Aware Deep Neural Networks (CADNN) for wind power simulation},
journal = {Array},
volume = {28},
pages = {100534},
year = {2025},
issn = {2590-0056},
doi = {https://doi.org/10.1016/j.array.2025.100534},
url = {https://www.sciencedirect.com/science/article/pii/S2590005625001614},
author = {Ali Forootani and Danial {Esmaeili Aliabadi} and Daniela Thr\"a n},
}

@article{gu2021s4,
  title={Efficiently modeling long sequences with structured state spaces},
  author={Gu, Albert and Goel, Karan and R{\'e}, Christopher},
  journal={arXiv preprint arXiv:2111.00396},
  year={2021}
}

@article{smith2022simplified,
  title={Simplified state space layers for sequence modeling},
  author={Smith, Jimmy TH and Warrington, Andrew and Linderman, Scott W},
  journal={arXiv preprint arXiv:2208.04933},
  year={2022}
}

@book{hamilton2020time,
  title={Time series analysis},
  author={Hamilton, James D},
  year={2020},
  publisher={Princeton university press}
}

@inproceedings{haoyietal-informer-2021,
  author    = {Haoyi Zhou and
               Shanghang Zhang and
               Jieqi Peng and
               Shuai Zhang and
               Jianxin Li and
               Hui Xiong and
               Wancai Zhang},
  title     = {Informer: Beyond Efficient Transformer for Long Sequence Time-Series Forecasting},
  booktitle = {AAAI},
  year      = {2021},
}

@article{BaiTCN2018,
  title={An empirical evaluation of generic convolutional and recurrent networks for sequence modeling},
  author={Bai, Shaojie and Kolter, J Zico and Koltun, Vladlen},
  journal={arXiv preprint arXiv:1803.01271},
  year={2018}
}

@inproceedings{wu2021autoformer,
  title={Autoformer: Decomposition Transformers with {Auto-Correlation} for Long-Term Series Forecasting},
  author={Haixu Wu and Jiehui Xu and Jianmin Wang and Mingsheng Long},
  booktitle={NeurIPS},
  year={2021}
}

@inproceedings{liu2021pyraformer,
  title={Pyraformer: Low-complexity pyramidal attention for long-range time series modeling and forecasting},
  author={Liu, Shizhan and Yu, Hang and Liao, Cong and Li, Jianguo and Lin, Weiyao and Liu, Alex X and Dustdar, Schahram},
  booktitle={ICLR},
  year={2021}
}

@inproceedings{
kim2022reversible,
title={Reversible Instance Normalization for Accurate Time-Series Forecasting against Distribution Shift},
author={Taesung Kim and Jinhee Kim and Yunwon Tae and Cheonbok Park and Jang-Ho Choi and Jaegul Choo},
booktitle={ICLR},
year={2022},
}

@inproceedings{zhou2022fedformer,
  title={{FEDformer}: Frequency enhanced decomposed transformer for long-term series forecasting},
  author={Zhou, Tian and Ma, Ziqing and Wen, Qingsong and Wang, Xue and Sun, Liang and Jin, Rong},
  booktitle={ICML},
  year={2022}
}

@article{gu2021combining,
  title={Combining Recurrent, Convolutional, and Continuous-time Models with Linear State-Space Layers},
  author={Gu, Albert and Johnson, Isys and Goel, Karan and Saab, Khaled and Dao, Tri and Rudra, Atri and R{\'e}, Christopher},
  journal={NeurIPS},
  year={2021}
}

@article{Liu2022NonstationaryTR,
  title={Non-stationary Transformers: Rethinking the Stationarity in Time Series Forecasting},
  author={Yong Liu and Haixu Wu and Jianmin Wang and Mingsheng Long},
  journal={arXiv preprint arXiv:2205.14415},
  year={2022}
}

@article{kuznetsov2020discrepancy,
  title={Discrepancy-based theory and algorithms for forecasting non-stationary time series},
  author={Kuznetsov, Vitaly and Mohri, Mehryar},
  journal={Annals of Mathematics and Artificial Intelligence},
  volume={88},
  number={4},
  pages={367--399},
  year={2020},
  publisher={Springer}
}

@article{nie2022time,
  title={A Time Series is Worth 64 Words: Long-term Forecasting with Transformers},
  author={Nie, Yuqi and Nguyen, Nam H and Sinthong, Phanwadee and Kalagnanam, Jayant},
  journal={arXiv preprint arXiv:2211.14730},
  year={2022}
}

@article{brunton2021modern,
  title={Modern Koopman theory for dynamical systems},
  author={Brunton, Steven L and Budi{\v{s}}i{\'c}, Marko and Kaiser, Eurika and Kutz, J Nathan},
  journal={arXiv preprint arXiv:2102.12086},
  year={2021}
}

@article{morton2018deep,
  title={Deep dynamical modeling and control of unsteady fluid flows},
  author={Morton, Jeremy and Jameson, Antony and Kochenderfer, Mykel J and Witherden, Freddie},
  journal={Advances in Neural Information Processing Systems},
  volume={31},
  year={2018}
}

@inproceedings{azencot2020forecasting,
  title={Forecasting sequential data using consistent koopman autoencoders},
  author={Azencot, Omri and Erichson, N Benjamin and Lin, Vanessa and Mahoney, Michael},
  booktitle={International Conference on Machine Learning},
  pages={475--485},
  year={2020},
  organization={PMLR}
}

@article{lusch2018deep,
  title={Deep learning for universal linear embeddings of nonlinear dynamics},
  author={Lusch, Bethany and Kutz, J Nathan and Brunton, Steven L},
  journal={Nature communications},
  volume={9},
  number={1},
  pages={4950},
  year={2018},
  publisher={Nature Publishing Group UK London}
}

@article{takeishi2017learning,
  title={Learning Koopman invariant subspaces for dynamic mode decomposition},
  author={Takeishi, Naoya and Kawahara, Yoshinobu and Yairi, Takehisa},
  journal={Advances in neural information processing systems},
  volume={30},
  year={2017}
}

@article{williams2015data,
  title={A data--driven approximation of the koopman operator: Extending dynamic mode decomposition},
  author={Williams, Matthew O and Kevrekidis, Ioannis G and Rowley, Clarence W},
  journal={Journal of Nonlinear Science},
  volume={25},
  pages={1307--1346},
  year={2015},
  publisher={Springer}
}

@book{absil2008optimization,
  title={Optimization algorithms on matrix manifolds},
  author={Absil, P-A and Mahony, Robert and Sepulchre, Rodolphe},
  year={2008},
  publisher={Princeton University Press}
}

@book{Mezic2017book,
	Author = {I. Mezi\'c},
	Publisher = {Springer},
	Title = {Spectral operator methods in dynamical systems: {T}heory and applications},
	Year = {2017}}

@article{Korda2016arxiv,
	Author = {Korda, M. and Mezi{\'c}, I.},
	Journal = {arXiv:1611.03537},
	Title = {Linear predictors for nonlinear dynamical systems: {K}oopman operator meets model predictive control},
	Year = {2016}}

@article{Brunton2016plosone,
	Author = {Brunton, S.\hspace{-.02in} L. AND Brunton, B.\hspace{-.02in} W. AND Proctor, J.\hspace{-.02in} L. AND Kutz, J.\hspace{-.02in}N.},
	Journal = {PLoS ONE},
	Number = {2},
	Pages = {e0150171},
	Title = {Koopman invariant subspaces and finite linear representations of nonlinear dynamical systems for control},
	Volume = {11},
	Year = {2016}}

@book{Kutz2016book,
	Author = {J. N. Kutz and S. L. Brunton and B. W. Brunton and J. L. Proctor},
	Publisher = {SIAM},
	Title = {Dynamic Mode Decomposition: Data-Driven Modeling of Complex Systems},
	Year = {2016}}

@article{chen19,
  title={Neural ordinary differential equations},
  author={Chen, Ricky TQ and Rubanova, Yulia and Bettencourt, Jesse and Duvenaud, David K},
  journal={Advances in Neural Information Processing Systems},
  volume={31},
  year={2018}
}

@inproceedings{vialard20,
	author = {Vialard, Fran\c{c}ois-Xavier and Kwitt, Roland and Wei, Suan and Niethammer, Marc},
	booktitle = {Advances in Neural Information Processing Systems},
	title = {A Shooting Formulation of Deep Learning},
	Year = {2020}
}

@article{Croston1972,
	Author = {Croston, J. D.},
	Day = {01},
	Issn = {1476-9360},
	Journal = {Journal of the Operational Research Society},
	Month = {Sep},
	Number = {3},
	Pages = {289--303},
	Title = {Forecasting and Stock Control for Intermittent Demands},
	Volume = {23},
	Year = {1972},
	}

@article{wen2017multi,
	Author = {Wen, Ruofeng and Torkkola, Kari and Narayanaswamy, Balakrishnan and Madeka, Dhruv},
	Journal = {arXiv preprint arXiv:1711.11053},
	Title = {A multi-horizon quantile recurrent forecaster},
	Year = {2017}}

@article{neurips_18_long_tail,
  title={Parsimonious quantile regression of financial asset tail dynamics via sequential learning},
  author={Yan, Xing and Zhang, Weizhong and Ma, Lin and Liu, Wei and Wu, Qi},
  journal={Advances in Neural Information Processing Systems},
  volume={31},
  year={2018}
}

@article{DIMOULKAS2019,
	Author = {Dimoulkas, I. and Mazidi, P. and Herre, L.},
	Journal = {International Journal of Forecasting},
	Number = {4},
	Pages = {1409 - 1423},
	Title = {Neural networks for {GEFCom2017} probabilistic load forecasting},
	Volume = {35},
	Year = {2019}}

@article{Saxena2019,
	Author = {Saxena, Harshit and Aponte, Omar and McConky, Katie T.},
	Journal = {International Journal of Forecasting},
	Number = {4},
	Pages = {1288 - 1303},
	Title = {A hybrid machine learning model for forecasting a billing period's peak electric load days},
	Volume = {35},
	Year = {2019},
}

@inproceedings{li2018diffusion,
  title={Diffusion Convolutional Recurrent Neural Network: Data-Driven Traffic Forecasting},
  author={Li, Yaguang and Yu, Rose and Shahabi, Cyrus and Liu, Yan},
  booktitle={International Conference on Learning Representations},
  year={2018}
}

@article{sen2019think,
  title={Think globally, act locally: A deep neural network approach to high-dimensional time series forecasting},
  author={Sen, Rajat and Yu, Hsiang-Fu and Dhillon, Inderjit S},
  journal={Advances in Neural Information Processing Systems},
  volume={32},
  year={2019}
}

@article{lecun1995convolutional,
	Author = {LeCun, Yann and Bengio, Yoshua},
	Journal = {The handbook of brain theory and neural networks},
	Number = {10},
	Title = {Convolutional networks for images, speech, and time series},
	Volume = {3361},
	Year = {1995}}

@article{hochreiter1997long,
	Author = {Hochreiter, Sepp and Schmidhuber, J{\"u}rgen},
	Journal = {Neural Computation},
	Number = {8},
	Pages = {1735--1780},
	Title = {Long short-term memory},
	Volume = {9},
	Year = {1997}}

@article{lv2014traffic,
	Author = {Lv, Yisheng and Duan, Yanjie and Kang, Wenwen and Li, Zhengxi and Wang, Fei-Yue},
	Journal = {IEEE Transactions on Intelligent Transportation Systems},
	Number = {2},
	Pages = {865--873},
	Title = {Traffic flow prediction with big data: {A} deep learning approach},
	Volume = {16},
	Year = {2014}}

@inproceedings{wang2019deepfactors,
	Author = {Wang, Yuyang and Smola, Alex and Maddix, Danielle and Gasthaus, Jan and Foster, Dean and Januschowski, Tim},
	Booktitle = {{International Conference on Machine Learning}},
	Pages = {6607--6617},
	Title = {Deep factors for forecasting},
	Year = {2019}}

@book{hyndman2018forecasting,
	Author = {Hyndman, Rob J and Athanasopoulos, George},
	Publisher = {OTexts},
	Title = {Forecasting: principles and practice},
	Year = {2018}}

@article{smyl2018m4,
	Author = {Slawek Smyl},
	Journal = {International Journal of Forecasting},
	Number = {1},
	Pages = {75-85},
	Title = {A hybrid method of exponential smoothing and recurrent neural networks for time series forecasting},
	Volume = {36},
	Year = {2020}}

@inproceedings{vaswani2017attention,
	Author = {Vaswani, Ashish and Shazeer, Noam and Parmar, Niki and Uszkoreit, Jakob and Jones, Llion and Gomez, Aidan N and Kaiser, {\L}ukasz and Polosukhin, Illia},
	Booktitle = {{Advances in Neural Information Processing Systems}},
	Pages = {5998--6008},
	Title = {Attention is all you need},
	Year = {2017}}

@article{salinas2020deepar,
  title={Deep{AR}: Probabilistic forecasting with autoregressive recurrent networks},
  author={Salinas, David and Flunkert, Valentin and Gasthaus, Jan and Januschowski, Tim},
  journal={International Journal of Forecasting},
  volume={36},
  number={3},
  pages={1181--1191},
  year={2020},
  publisher={Elsevier}
}

@article{kingma2014adam,
	Author = {Kingma, Diederick P and Ba, Jimmy},
	Booktitle = {International Conference on Learning Representations (ICLR)},
	Title = {Adam: A method for stochastic optimization},
	Year = {2015}}

@article{callot2017modeling,
	Author = {Callot, Laurent AF and Kock, Anders B and Medeiros, Marcelo C},
	Journal = {Journal of Applied Econometrics},
	Number = {1},
	Pages = {140--158},
	Publisher = {Wiley Online Library},
	Title = {Modeling and forecasting large realized covariance matrices and portfolio choice},
	Volume = {32},
	Year = {2017}}

@article{makridakis2021m5,
  title={The M5 Uncertainty competition: Results, findings and conclusions},
  author={Makridakis, Spyros and Spiliotis, Evangelos and Assimakopoulos, Vassilios and Chen, Zhi and Gaba, Anil and Tsetlin, Ilia and Winkler, Robert L},
  journal={International Journal of Forecasting},
  year={2021},
  publisher={Elsevier}
}

@inproceedings{informer,
  author    = {Haoyi Zhou and
               Shanghang Zhang and
               Jieqi Peng and
               Shuai Zhang and
               Jianxin Li and
               Hui Xiong and
               Wancai Zhang},
  title     = {Informer: Beyond Efficient Transformer for Long Sequence Time-Series Forecasting},
  booktitle = {The Thirty-Fifth {AAAI} Conference on Artificial Intelligence},
  volume    = {35},
  number    = {12},
  pages     = {11106--11115},
  year      = {2021},
}

\appendix

\subsection{Pressure Surface MSE Test}

\begin{figure*}
    \centering
    \includegraphics[width=0.95\linewidth]{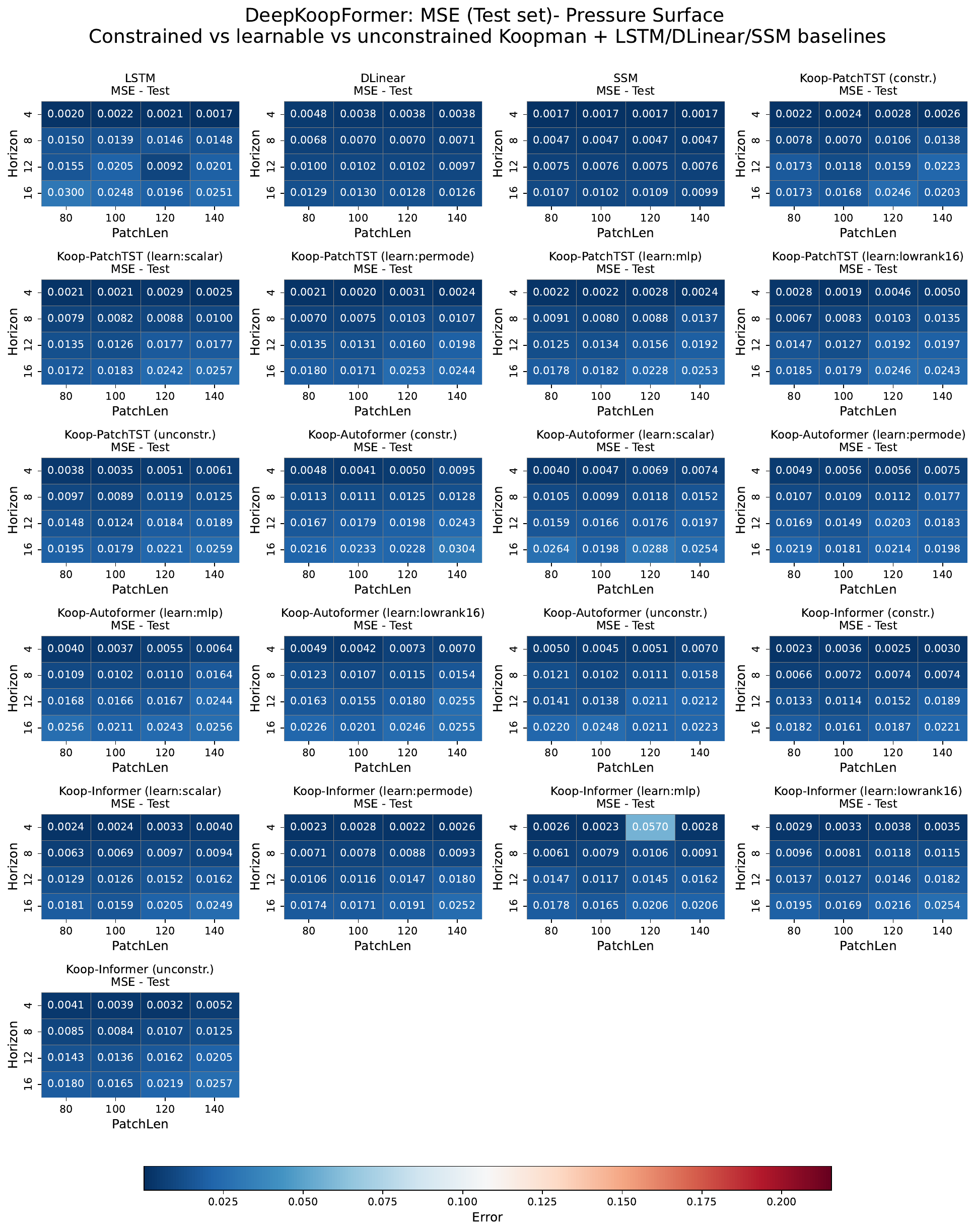}
    %\caption{Caption}
    \label{fig:pressure_mse_error}
\end{figure*}

\subsection{Wind Speed MSE Test}

\begin{figure*}
    \centering
    \includegraphics[width=0.95\linewidth]{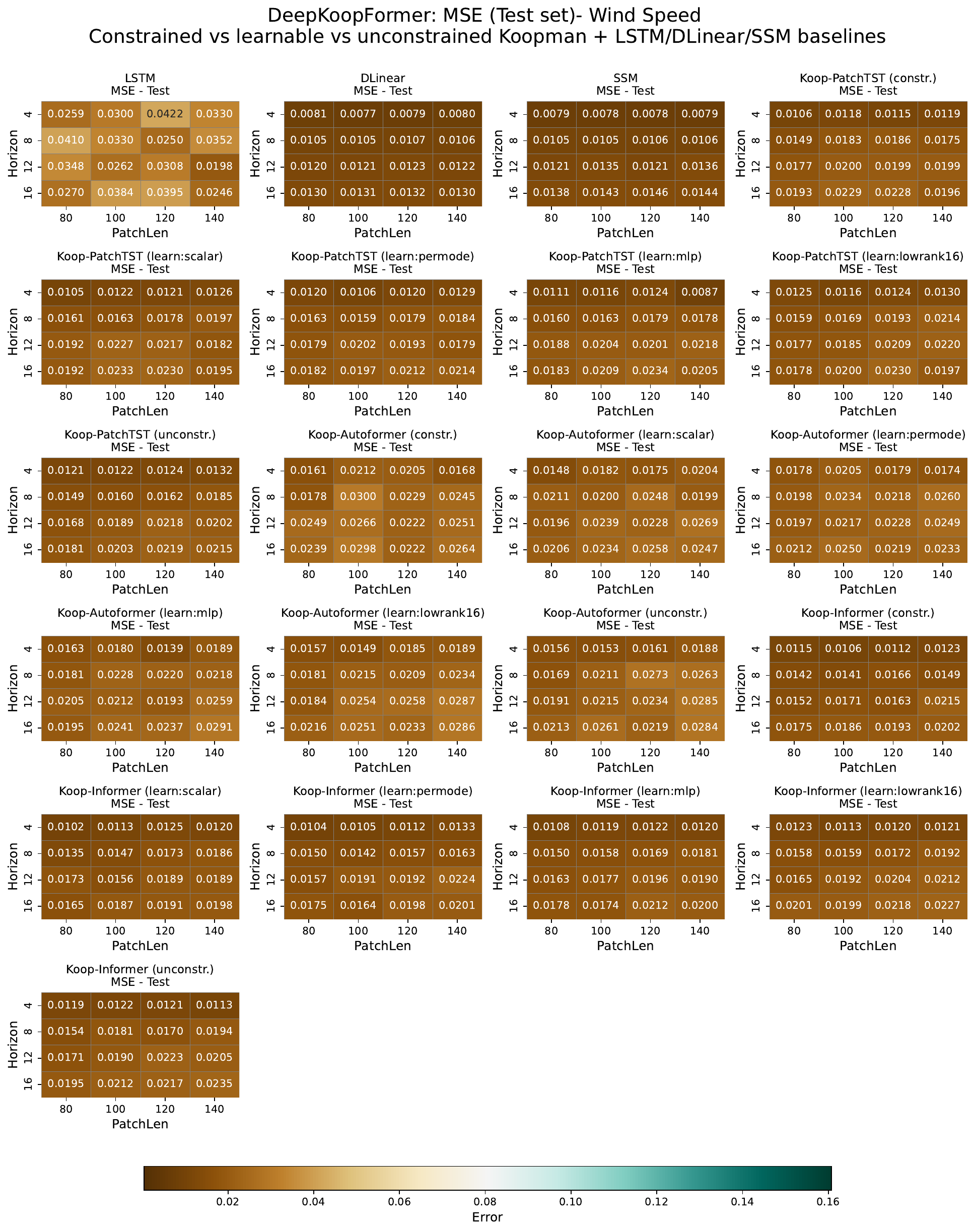}
    %\caption{Caption}
    \label{fig:wind_mse_error}
\end{figure*}

\subsection{Cryptocurrency}

\begin{figure*}
    \centering
    \includegraphics[width=0.95\linewidth]{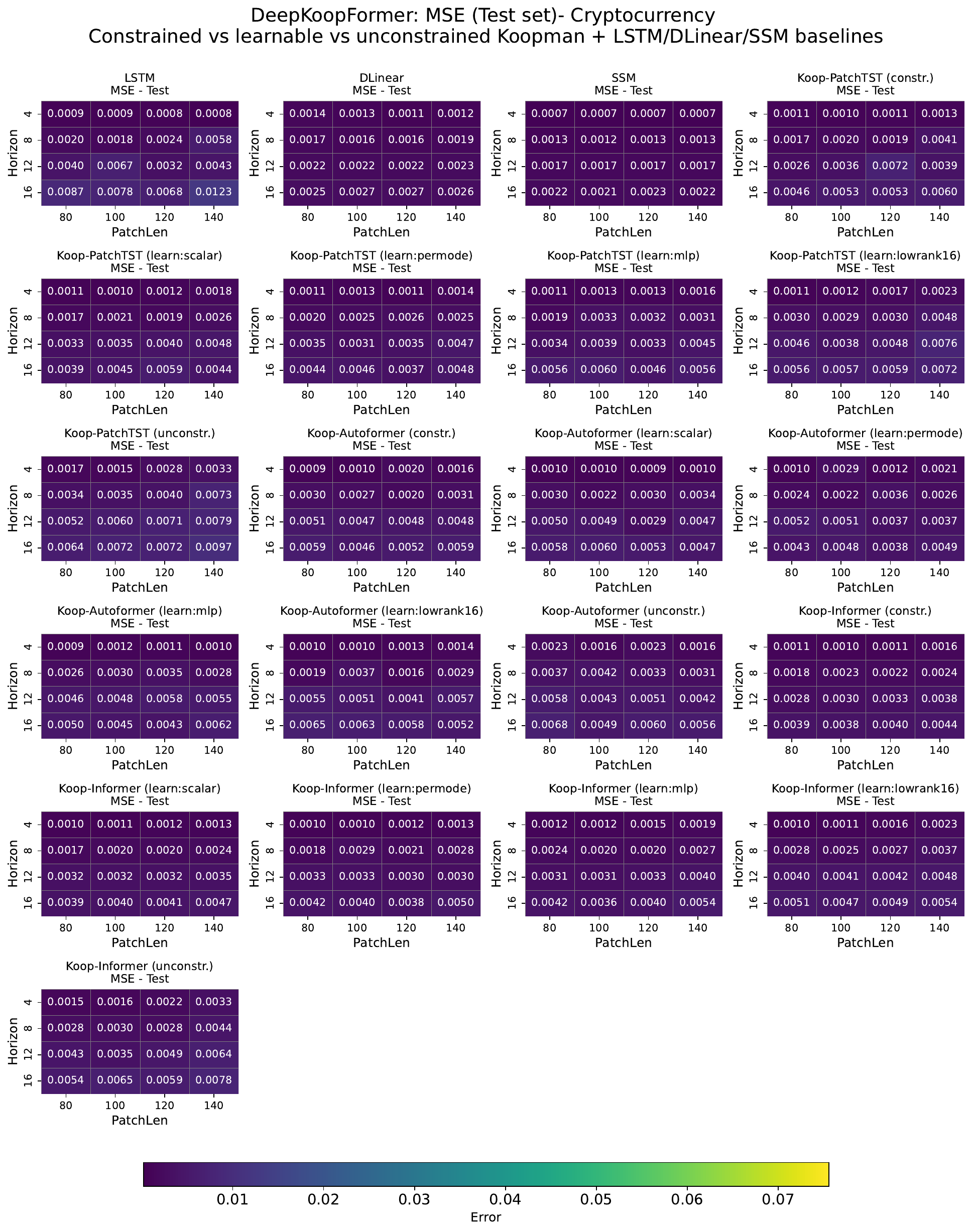}
    %\caption{Caption}
    \label{fig:crypto_mse_error}
\end{figure*}

\subsection{Energy Systems MSE Test}

\begin{figure*}
    \centering
    \includegraphics[width=0.99\linewidth]{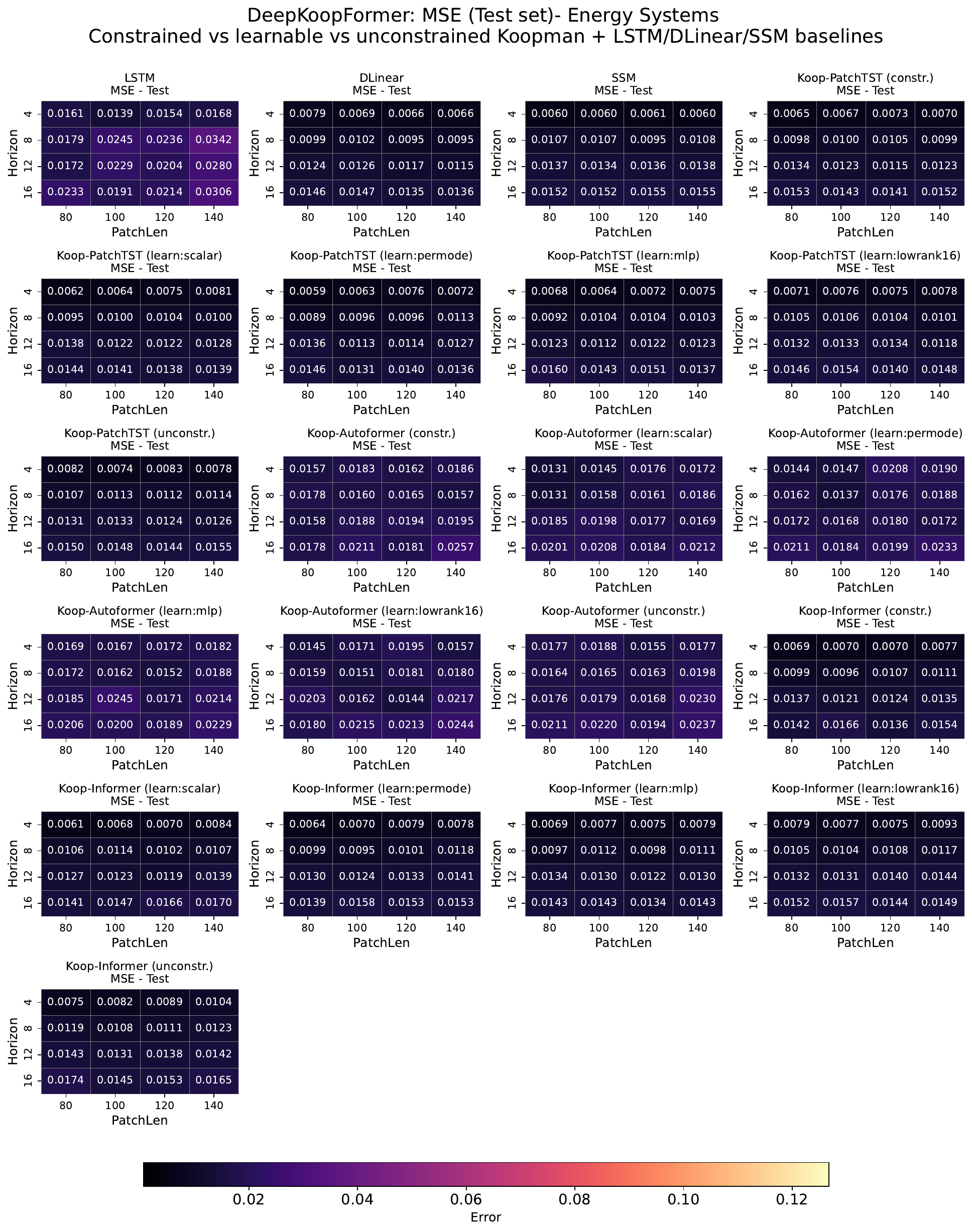}
    %\caption{Caption}
    \label{fig:energy_mse_error}
\end{figure*}

\subsection{ERA5 Wind Speed MSE Test}

\begin{figure*}
    \centering
    \includegraphics[width=0.99\linewidth]{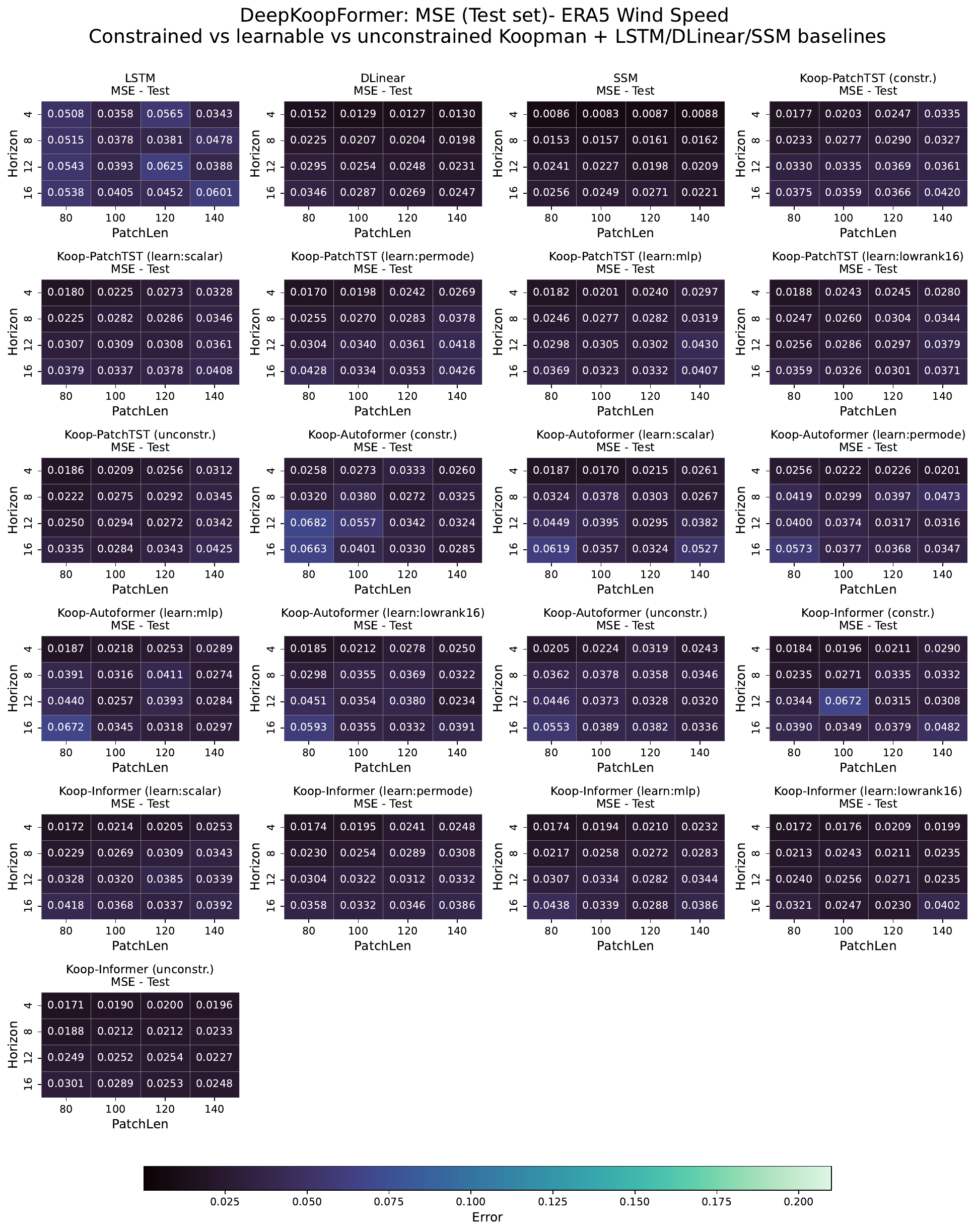}
    %\caption{Caption}
    \label{fig:era5_mse_error}
\end{figure*}

\newpage

%% Biography
%\begin{IEEEbiography}{First A. Author}{\space}(M'76--SM'81--F'87) 
%\end{IEEEbiography}

\end{document}